\documentclass[journal]{IEEEtran}
\usepackage{amsmath,amsfonts}
\usepackage{algorithmic}
\usepackage{algorithm}
\usepackage{xcolor}
\usepackage[export]{adjustbox} 
\definecolor{lightblue}{rgb}{0.0, 0.4, 0.902}
\definecolor{darkblue}{rgb}{0.0, 0.0, 0.545}
\definecolor{lightgreen}{rgb}{0.0, 0.59, 0.0}
\definecolor{darkgreen}{rgb}{0.0, 0.40, 0.0}
\definecolor{lightyellow}{rgb}{0.996, 0.834, 0.597}
\definecolor{darkred}{rgb}{0.80, 0.08, 0.18}
\definecolor{lightlightgreen}{rgb}{0, 0.79, 0.0}
\definecolor{yellowgreen}{rgb}{0.5, 0.5, 0.0}
\usepackage{rotating}

\usepackage{array}
\usepackage[caption=false,font=scriptsize,labelfont=sf,textfont=sf]{subfig}
\usepackage{textcomp}
\usepackage{stfloats}
\usepackage{url}
\usepackage{verbatim}
\usepackage{graphicx}
\usepackage{cite}
\usepackage{bm}
\usepackage{multirow}
\usepackage{threeparttable}
\usepackage{booktabs}
\usepackage{float}
\usepackage{color,xcolor}
\definecolor{light-gray}{gray}{0.95}
\usepackage{amsthm}
\usepackage{etoolbox}
\makeatletter
\patchcmd{\@makecaption}
{\scshape}
{}
{}
{}
\makeatother
\newtheorem{theorem}{Theorem}  

\newtheorem{remark}{Remark}
\newtheorem{assumption}{Assumption}
\newtheorem{corollary}{Corollary}

\begin{document}

\title{FX-DARTS: Designing Topology-unconstrained Architectures with Differentiable Architecture Search and Entropy-based Super-network Shrinking}


\author{Xuan Rao, Bo Zhao,~\IEEEmembership{Senior Member,~IEEE}, Derong Liu,~\IEEEmembership{Fellow,~IEEE}, Cesare Alippi,~\IEEEmembership{Fellow,~IEEE}

\thanks{This work was supported in part by the National Key Research and Development Program of China under Grant 2018AAA0100203, in part by the National Natural Science Foundation of China under Grants 61973330, 62073085 and  62350055, in part by the Shenzhen Science and Technology Program under Grant JCYJ20230807093513027, in part by the Fundamental Research Funds for the Central Universities under grant 1243300008, and in part by the Beijing Normal University Tang Scholar. (Corresponding author: Bo Zhao)}
\thanks{Xuan Rao and Bo Zhao are with the School of Systems Science, Beijing Normal University, Beijing 100875, China (e-mail: raoxuan98@mail.bnu.edu.cn; zhaobo@bnu.edu.cn).}
\thanks{Derong Liu is with the School of System Design and Intelligent Manufacturing, Southern University of Science and Technology, Shenzhen 518000, China (email: liudr@sustech.edu.cn), and also with the Department of Electrical and Computer Engineering, University of Illinois Chicago, Chicago, IL 60607, USA (e-mail: derong@uic.edu).}
\thanks{Cesare Alippi is with the Dipartimento di Elettronica e Informazione, Politecnico di Milano, Milano 20133, Italy (e-mail: cesare.alippi@polimi.it), and also with the Universita’ Della Svizzera Italiana, Lugano, Switzerland (email: cesare.alippi@usi.ch).}

\thanks{© 2025 IEEE. This manuscript is submitted to IEEE Transaction on Neural Network and Learning Systems and is under reviewed. Personal use of this manuscript is permitted. Permission from IEEE must be obtained for all other uses, in any current or future media, including reprinting/republishing this material for advertising or promotional purposes, creating new collective works, for resale or redistribution to servers or lists, or reuse of any copyrighted component of this work in other works.}
}

\markboth{Submitted to IEEE transaction on neural network and learning systems}%
{Rao \MakeLowercase{\textit{et al.}}: FX-DARTS: Exploring Topology-unconstrained Architectures with Differentiable Architecture Search and Entropy-based Super-network Shrinking}

\maketitle

\begin{abstract}
Strong priors are imposed on the search space of Differentiable Architecture Search (DARTS), such that cells of the same type share the same topological structure and each intermediate node retains two operators from distinct nodes. While these priors reduce optimization difficulties and improve the applicability of searched architectures, they hinder the subsequent development of automated machine learning (Auto-ML) and prevent the optimization algorithm from exploring more powerful neural networks through improved architectural flexibility. This paper aims to reduce these prior constraints by eliminating restrictions on cell topology and modifying the discretization mechanism for super-networks. Specifically, the Flexible DARTS (FX-DARTS) method, which leverages an Entropy-based Super-Network Shrinking (ESS) framework, is presented to address the challenges arising from the elimination of prior constraints. Notably, FX-DARTS enables the derivation of neural architectures without strict prior rules while maintaining the stability in the enlarged search space. Experimental results on image classification benchmarks demonstrate that FX-DARTS is capable of exploring a set of neural architectures with competitive trade-offs between performance and computational complexity within a single search procedure.
\end{abstract}


\section{Introduction}
Over the past decade, the powerful representation ability of deep neural networks (DNNs) has contributed to significant progress in various machine learning tasks, including computer vision \cite{szegedy2015going, gao2021structure}, natural language processing \cite{devlin2018bert, otter2020survey}, system identification and control \cite{chen2018optimal, yang2018hierarchical}, time series prediction \cite{bandara2020lstm}, and autonomous vehicles \cite{gao2018object, gao2021interacting, gao2021trajectory}, among others. Undoubtedly, the architecture design of neural networks plays a pivotal role in these breakthroughs \cite{hochreiter1997long, liu2006motif, li2016deep, chollet2017xception, yu2015multi}. To address the labor-intensive and time-consuming trial-and-error process of DNN architecture design, neural architecture search (NAS) \cite{he2021automl} has emerged as a promising approach. NAS automates the exploration of a vast space of potential architectures, traditionally through three key steps: defining a search space, selecting a search algorithm, and identifying an optimal architecture within the search space. The effectiveness of NAS heavily relies on the careful design of both the search space and the search strategy, as a well-constructed search space can significantly enhance the search algorithm's ability to discover optimal neural architectures \cite{radosavovic2020designing}.

The primary focus of this paper is to address the observation that most cell-based Neural Architecture Search (NAS) methods impose stringent prior rules on the designed architectures, which indicates that Automated Machine Learning (Auto-ML) has not yet to achieve full automation. Drawing a parallel to the evolution from manual feature engineering to end-to-end learning in deep learning (DL), we argue that fewer prior constraints should be imposed on the architectures explored by the NAS algorithms. 

\begin{figure}
	\centering
	\subfloat[Search space comparison]{
		\includegraphics[width=0.8\linewidth]{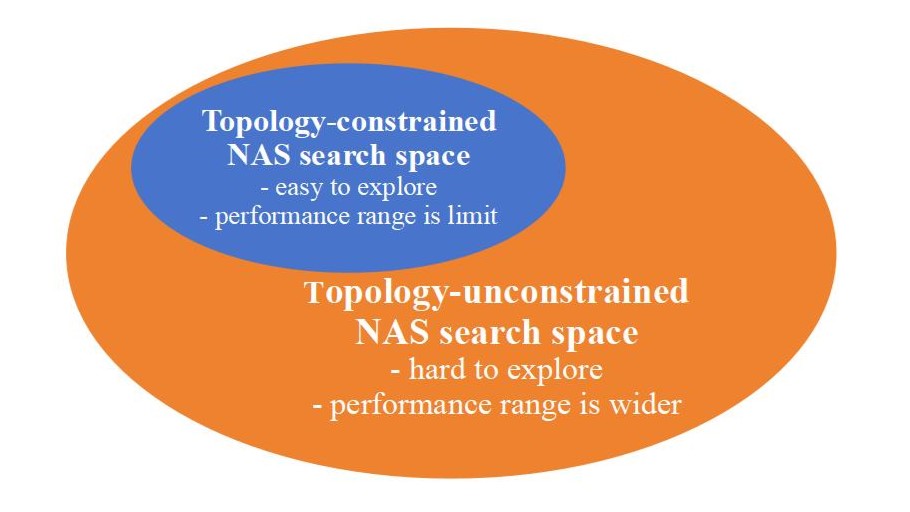}
		\label{fig:search_space}
	}
	\\
	
	\subfloat[DARTS-V2]{
		\includegraphics[height=0.25\textheight]{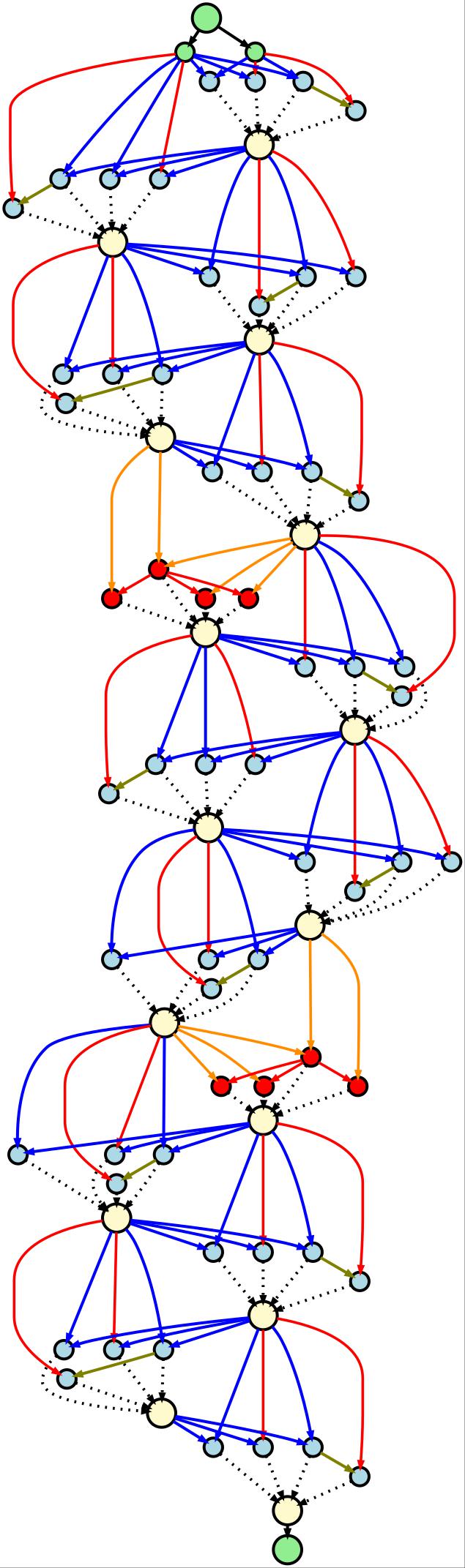}
		\label{fig:darts_cifar}
	}
	\subfloat[CDARTS]{
		\includegraphics[height=0.25\textheight]{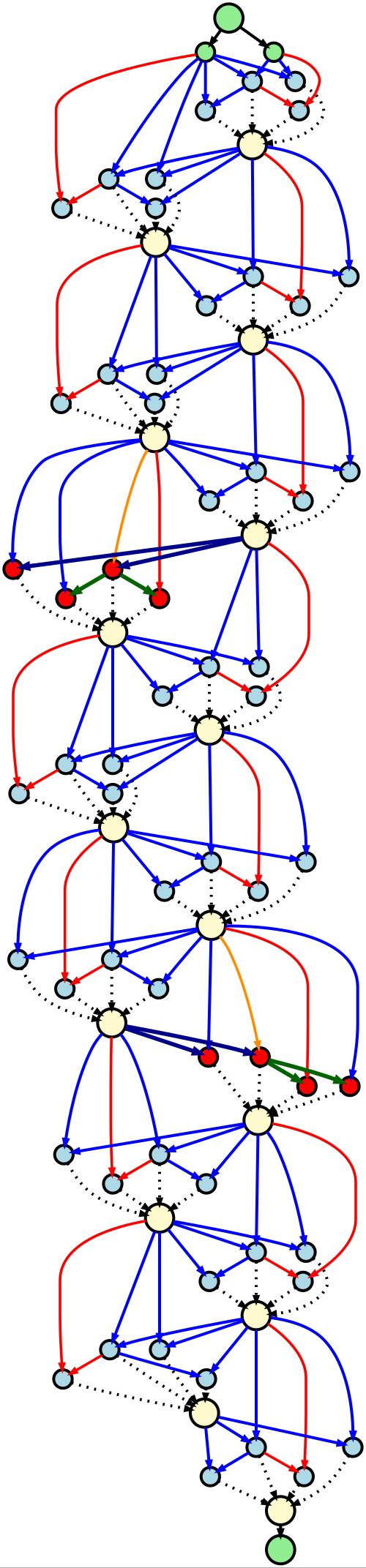}
		\label{fig:cdarts_cifar}
	}
	\subfloat[FX-DARTS]{
		\includegraphics[height=0.25\textheight]{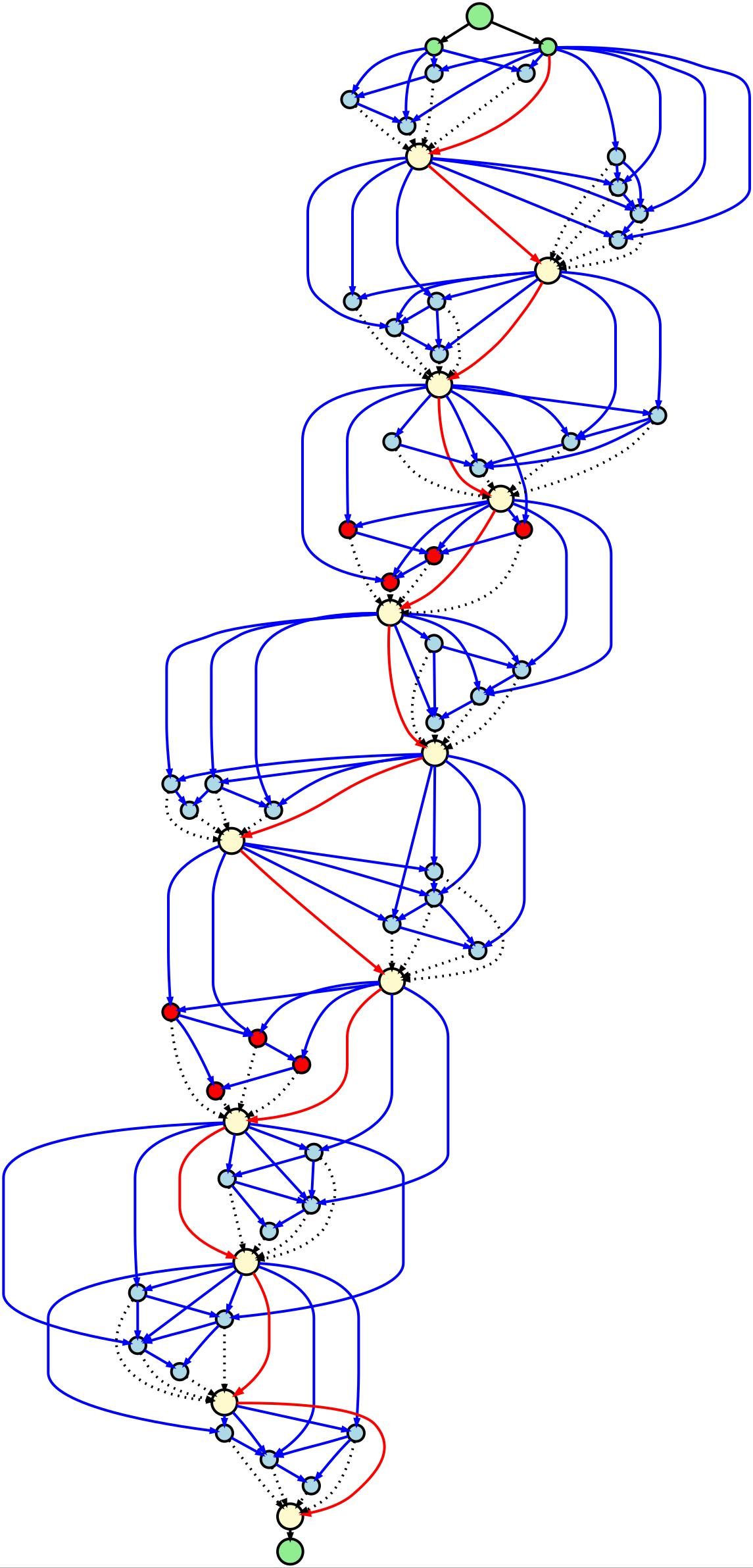}
		\label{fig:genotype_48epoch}
	}
	\subfloat[FX-DARTS]{
		\includegraphics[height=0.25\textheight]{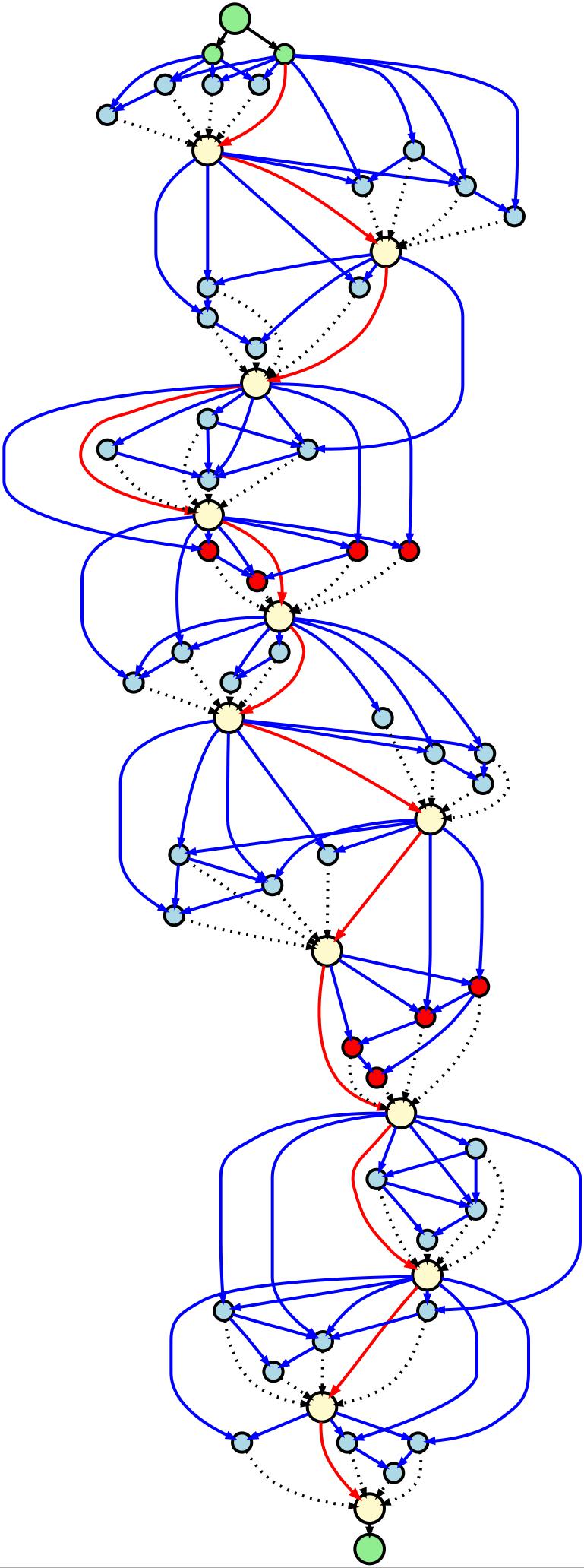}
		\label{fig:genotype_64epoch}
	}
	\caption{Visualization of topology-constrained architectures and our flexible architectures derived on the CIFAR dataset. (a) Search space comparison. Specifically, topology-constrained architectures are included in the topology-unconstrained search space. (b) and (c) Architectures derived in the topology-constrained search space by DARTS and CDARTS, respectively. These methods assume that cells of the same type share the same structure. (d) and (e) Architectures derived by our FX-DARTS in the topology-unconstrained search space in one NAS procedure. Particularly, all cells have their unique structures.}
	\label{fig:dartscifar}
\end{figure}

To this end, this paper explores neural architectures with more flexible structures under the framework of Differentiable Architecture Search (DARTS) \cite{liu2018darts} (the comparisons between topology-constrained and unconstrained search spaces are depicted in Fig. \ref{fig:search_space}). Our approach offers several key benefits. First, it reduces the imposition of prior constraints on architecture topology, which aligns with the progressive development of Auto-ML. Moreover, an enlarged search space provides the potential for discovering more powerful neural networks with preferable trade-offs between representation performance and computational complexity.

Formally, DARTS transforms the discrete search space into a continuous one by replacing the categorical selection of a specific operator between nodes with a weighted combination of all possible operators. The weights for these operators are parameterized by a set of architectural parameters. Under this framework, both model parameters (neural network weights) and architectural parameters (architecture design) are optimized jointly in an end-to-end manner using gradient descent, enabling the efficient architecture search. Despite the demonstrated efficacy of DARTS-based approaches, it remains some drawbacks worth considering, which are partially addressed in this work.
\begin{itemize}[]
	\item As depicted in Fig. \ref{fig:darts_cifar} and \ref{fig:cdarts_cifar} , DARTS-like approaches construct architectures by interleaving neural cells, where cells of the same type (normal or reduction cells) share identical topological structures. Although this constrained search space improves stability and portability, it restricts structural flexibility, and potentially limits advancements in representation performance and computational efficiency. Importantly, there is no guarantee that an architecture with such a constrained structure is better than that with a unconstrained structure.
	\item The discrete architectures derived from DARTS are governed by strong prior rules. For instance, each intermediate node is restricted to preserve two operators only from distinct nodes. While the constraint simplifies the optimization process, it narrows the diversity of architectures that can be explored. As a result, the architecture search is biased toward structures that conform to these predefined rules.
	\item In particular, DARTS-based approaches optimize model parameters and architectural parameters alternately, which results in the co-adaptation of model training and architecture selection.  
\end{itemize}

The inherent constraints of cell topology and the discretization mechanism for the super-network significantly limit the architectural flexibility. While removing these inductive biases could theoretically enhance the design freedom, but two critical challenges emerge naturally. First, without topological constraints and operator retention rules, ambiguity arises in determining the types and quantities of operators to preserve for each computing node during the super-network discretization process. Additionally, the removal of these constraints exponentially expands the search space, leading to optimization instability and amplifying the adaptation gap between super-network training and architecture optimization. 

To address these challenges, we propose the Flexible DARTS (FX-DARTS) method, which primarily incorporates the Entropy-based Super-network Shrinking (ESS) framework. Our approach introduces two key innovations. First, to decouple the co-adaptation between model parameters and architectural parameters, we implement a distinct separation between super-network training and architecture optimization. Second, to effectively derive candidate architectures from the super-network, FX-DARTS employs a progressive discretization process that biases the super-network toward sparser structures. This controlled shrinking mechanism facilitates the discovery of diverse neural architectures with varying computational complexities within a single search procedure. The main contributions of this work are summarized as follows.

\begin{enumerate}
	\item The imposition of prior constraints on cell topology and the discretization mechanism for the super-network is minimized to enable the derivation for more flexible architectures which potentially have more powerful representation capabilities.
	\item The FX-DARTS method is introduced to effectively address the challenges associated with the removal of prior constraints. Specifically, our method decouples the optimization processes of super-network training and architecture selection to avoid the co-adaptation for model training and architecture optimization. In addition, it enables the elegant derivations of discrete architectures through progressive super-network shrinking without relying on strong prior constraints.
	\item Extensive experimental evaluations on image classification benchmarks demonstrate that FX-DARTS achieves a superior trade-off between the performance and computational complexity compared to other NAS architectures. Notably, FX-DARTS attains 76.4\% top-1 accuracy on ImageNet-1K with 5.1M parameters and 610M FLOPs, which outperforms numerous state-of-the-art topology-constrained NAS methods while maintaining efficient search costs of merely 3.2 GPU-hours.
\end{enumerate}

\section{Related Works}
\subsection{Neural architecture search}
Aiming to automate the design process of neural networks, NAS has attracted extensive attention from researchers and has been applied to different DL domains. Early NAS methods are mainly developed based on Reinforcement Learning (RL) \cite{zoph2018learning, baker2016designing, pham2018efficient} and Evolutionary Algorithms (EAs) \cite{liu2017hierarchical, real2019regularized}. Operating primarily in discrete search spaces, these methods sample and train hundreds of candidates from scratch to evaluate the validation accuracy to learn a meta-controller for architecture sampling. However, the extensive computational costs prohibit the generalization of these methods to more challenging NAS tasks. Some efforts have been made to accelerate the search process of NAS based on the strategy of weight-sharing \cite{pham2018efficient}.

In addition, some approaches, which can roughly be classified into Latent Space Optimization (LSO) \cite{luo2018neural} and differentiable NAS \cite{liu2018darts}, achieve architecture search in a continuous space. Regarding NAS tasks as discrete optimization problems of Directed Acyclic Graphs (DAGs), LSO-based approaches train deep generative models to learn the continuous representations of architectures, transforming the search space from discrete to continuous \cite{DBLP:conf/nips/YanZAZ020}. In this way, neural architectures are searched in the continuous latent space and evaluated in the discrete space. For example, the DAG Variational Autoencoder (D-VAE) is proposed to learn the continuous representations of neural architectures \cite{DBLP:conf/nips/ZhangJCGC19}, such that architectures can be optimized in a continuous space by Bayesian optimization. Convexity Regularized LSO (CR-LSO) utilizes the guaranteed convexity of input convex neural networks to regularize the learning process of the latent space, obtaining an approximately convex mapping of architecture performance \cite{rao2022cr}.  

Coupled with the strategy of weight-sharing, differentiable NAS transforms the discrete optimization of NAS into continuous optimization using techniques such as soft attentions \cite{liu2018darts} and reparameterization tricks \cite{xie2018snas}. As the most representative method, DARTS reduces the search cost of NAS to several GPU-hours, which has attracted a lot of attention in the literature \cite{chen2021progressive, bi2020gold, tian2021discretization, singh2024kl, jing2023architecture, xue2023improved, chen2023mngnas, 9432795}. Numerous follow-up works focus on overcoming the instability of DARTS \cite{chu2020fair, chu2020darts, chu2020noisy}.

Despite the effectiveness, some evidence shows that the search space of DARTS plays an important role, such that even an architecture derived by random search achieves competitive performance \cite{li2020random}. Furthermore, it is reported that the use of training tricks in the evaluation phase plays a predominant role in the reported performance of searched architectures in DARTS-based methods. Also, the cell-based search space of DARTS exhibits a very narrow accuracy range, making the degree of discrimination among the architectures searched by different state-of-the-art methods not obvious enough \cite{DBLP:conf/iclr/YangEC20}. 

There are weight-sharing NAS methods which differ from the variants of DARTS \cite{stamoulis2019single, cai2018proxylessnas, su2021k}. ProxylessNAS uses path-level pruning with binary gates to reduce GPU memory consumption, which enables the direct search on ImageNet without proxy tasks \cite{cai2018proxylessnas}. It also introduces the latency modeling through a differentiable loss to optimize hardware-specific performance. However, its search space is limited to 648 candidates, since it restricts hierarchical interactions between non-sequential nodes. To reduce the weight co-adaption problem in weight-sharing NAS and improve the NAS efficiency, the single-path NAS is suggested by sampling candidate operators between sequential nodes with a uniform sampling strategy \cite{stamoulis2019single}. To enable more flexible weight sharing strategy across sub-networks in weight-sharing NAS, K-shot NAS is established by assigning the operator weights of different paths by a convex combination of several independent meta-weights \cite{su2021k}. Similar to ProxylessNAS, both single-path and K-shot NAS focus on sequential node interactions and lack mechanisms for hierarchical connections between non-sequential blocks, which constraint the architectural flexibility.

\subsection{Flexible architecture search}
Currently, most cell-based NAS methods rely on the topology-sharing strategy, where normal cells or reduction cells share the same topological structure. The motivation behind this strategy is reasonable since it reduces the difficulty of searching an architecture by focusing on micro neural cells. This complexity reduction has been verified by many manually designed or automated architectures, such as ResNet \cite{DBLP:conf/cvpr/HeZRS16}, ShuffleNet \cite{DBLP:conf/cvpr/ZhangZLS18}, NASNet \cite{baker2016designing}, DARTS \cite{liu2018darts}, ViT \cite{DBLP:conf/iclr/DosovitskiyB0WZ21}, etc. Some evidence also shows that decoupling the cell-based NAS into operator search for edges and topology search for cells can further improve the performance of DARTS \cite{9412285, DBLP:conf/cvpr/GuW0YWLC21}. This implies that decoupled search spaces, which are smaller, assist to discover high-performance architectures more easily. Unfortunately, current research still relies on the topology-sharing strategy to avoid the direct search for an architecture.

Nevertheless, the basic building blocks of these topology-constrained architectures, namely the neural cell for brevity, become component-diverse and topology-flexible. This phenomenon presents a trend to design neural architectures with flexible structures. Furthermore, we emphasize that the topology-sharing strategy implemented by these methods goes against the successive development of data-driven DL and Auto-ML, since learning systems have been imposed strong priors by human experts. Additionally, it is not reasonable to assume that an optimal architecture in an constrained space is also optimal in an unconstrained space. Note that architectures in the constrained space are also included in the unconstrained space. Thus, if a topology-constrained architecture performs well on a given task, there are potentially better architectures in the unconstrained space. Now, the problem becomes how NAS methods can overcome the challenges caused by the exponential enlargement of the architecture design space, which is partially answered in this paper. Furthermore, we emphasize that the attempt to search neural architecture in an unconstrained space is important for the NAS community, as it aligns with the successive development of Auto-ML. Analogous to the transition from feature engineering to end-to-end DL, fewer priors from human experts will be imposed on the searched architectures in NAS.

In contrast to investigating topology-sharing search spaces, some efforts have been made to derive neural architectures with more flexible structures. For example, a novel compact group neural architecture has been empirically demonstrated to be more effective than popular architectures for the classification task of noise-corrupted images \cite{9216490}.
However, the proposed model is not applicable to large representation learning tasks, such as ImageNet classification. RandWire-WS utilizes a stochastic network generator to create randomly wired neural networks, and the evaluation results show that a flexible architecture achieves the top-1 accuracy of 74.7\% on ImageNet, which is higher than those of DARTS architectures \cite{DBLP:conf/iccv/XieKGH19}. GOLD-NAS attempts to stabilize the one-level optimization of DARTS with the data augmentation technique and to derive flexible architectures with progressive pruning of the super-network \cite{bi2020gold}. However, the main operator space employed by GOLD-NAS contains only two operators. One of the operators is the $3\times3$ depthwise separable convolution, which has been demonstrated to be the most effective operator in the DARTS search space \cite{9412285}. Differentiable Neural Architecture Distillation (DNAD) utilizes the distilled knowledge of a well-trained network to guide the optimization of the super-network in the unconstrained space, and the results show that distillation has stronger regularization impacts than the data augmentation technique \cite{rao2022dnad}. Nevertheless, the mechanism of DNAD is complicated, since some hyper-parameters require careful selection.

There are 

\section{Preliminary}
\subsection{Cell-based Convolution Neural Architectures}
Formally, an architecture is composed of $L$ neural cells, where the cell $k$ $(1 \le k \le L)$ can be abstracted as a DAG with $N$ nodes. For the sake of simplicity, the cell notation $k$ is omitted temporally until the proposed FX-DARTS is introduced. In NAS, the topological structures are the main concerns regardless of the inner model parameters. 

In the DAG which represents a cell, each node denotes a layer in the neural network, and each directed edge $(i,j)$ is associated with a specific operator $f^{o}_{(i,j)}$ that propagates information from layers $i$ to $j$. Specifically, the operator type $o$ for $f^{o}_{(i,j)}$ (e.g., the skip connection, the $3\times3$ depthwise separable convolution and zero operation) is selected from a pre-defined operator space $\mathcal{O}$. The trainable parameters of $f^{o}_{(i,j)}$ is denoted as $\theta_{(i,j)}^{o}$.

For an arbitrary cell, the first two nodes $x_1$ and $x_{2}$ represent the inputs of current cell. They are selected as the outputs of previous two cells, respectively. The other nodes except the last one, i.e., $x_{3}$ to $x_{N-1}$, are computing nodes. The value of a computing node $x_{j}\ (2 < j < N) $ is the sum of information propagated by two specific operators from two distinct nodes 
\begin{align}
	x_{j} = f_{(i,j)}^o\left(x_{i}\right) + f_{(i',j)}^{o'} \left(x_{i'}\right),
\end{align}
where $i \in \mathcal{N}(j) = \{ 1 \le i < j \}$ is the previous nodes of node $j$ of current cell and $i'\neq j$. The output of a cell is the concatenation of the outputs of all intermediate nodes as
\begin{align}
	x_{N} = {\rm concat}\left([x_{3},x_{4},...,x_{N-1}], {\rm dim}=1 \right),
\end{align}
where PyTorch-like code is utilized to express the formulation. 

Formally, given the cell number $L$, the node number $N$ and the operator space $\mathcal{O}$ as the pre-defined hyper-parameters before architecture search, the goal of DARTS is to choose two input edges $f_{(i,j)}^{o}$ and $f_{(i',j)}^{o'}$ from two distinct nodes $i$ and $i'$ for each computing node $j$. As it can be seen, the selection for an input edge $f_{(i,j)}^{o}$ involves two considerations, namely, the information source $i \in \mathcal{N}(j)$ and the operator type $o \in \mathcal{O}$.  

\subsection{Continuous Relaxation of Search Space}
The intrinsic goal of DARTS is to solve a discrete graph optimization problem, which involves to determine the optimal topological structures for neural cells. To make the optimization process of NAS differentiable, DARTS formulates the information propagation from nodes $i$ to $j$ as a weighted sum of all operators in $\mathcal{O}$ as
\begin{align}
f_{(i,j)}^{\tilde{o}} \left(x_{i} \right) = \sum_{o\in\mathcal{O}} a_{(i,j)}^{o} \cdot f_{(i,j)}^{o}\left(x_{i}\right),
\end{align}
where 
\begin{equation}
	a_{(i,j)}^{o} = {\rm Softmax}\left(\mathbf{\alpha}_{(i,j)}\right)_{o} = \frac{{\rm exp}\left(\alpha^{o}_{\left(i,j\right)}\right)}{\sum_{o'\in \mathcal{O}}{\rm exp}\left(\alpha_{\left(i,j\right)}^{o'}\right)}
\end{equation}
is the contribution weight of the operator $o$ in edge $(i,j)$, and $\alpha_{(i,j)}^{o}$ is the meta-parameter to be optimized which determines the contribution for operator $o$ in edge $(i,j)$. Then, the value for computing node $j$ is reformulated
\begin{align}
	x_{j} = \sum_{i \in \mathcal{N}(j)} f_{(i,j)}^{\tilde{o}}\left(x_{i}\right).
\end{align}
The set of meta-parameters ${\bm{\alpha}} = \{ \alpha_{(i,j)}^{o} \}$ for $1 \le i < j < N$ and $o \in \mathcal{O}$ is called architectural parameters (note that the cell notation $k$ is omitted). Similarly, the set of trainable parameters $\bm{\theta} = \{ \theta_{(i,j)}^{o} \}$ for these operators and some pre-defined layers is called model parameters. 

By replacing the information propagation between two nodes by the weighted mixture of operators, a super-network $f_{(\bm{\alpha},\bm{\theta})}^{\rm Super} \left( \textbf{x} \right)$ whose sub-graphs are the candidate neural architectures is constructed. Due to the continuous relaxation of operator selections, the architectural parameters $\bm{\alpha}$ and model parameters $\bm{\theta}$ of $f_{(\bm{\alpha},\bm{\theta})}^{\rm Super} \left( \textbf{x} \right)$ can be optimized jointly or alternately in an end-to-end manner. Let $\mathcal{D}_{\rm train}$ and $\mathcal{D}_{\rm valid}$ be the training and validation sets, respectively, and 
\begin{align}
	p_{{\bm{\alpha}}, \bm{\theta}}\left( \mathbf{y} \vert \mathbf{x} \right) = {\rm Softmax}\left( f_{(\bm{\alpha},\bm{\theta})}^{\rm Super} \left( \textbf{x} \right) \right)
\end{align}
be the predictive distribution of the neural network given the supervised sample $( \mathbf{x}, \mathbf{y})$. In DARTS-like approaches, the optimization for $\bm{\alpha}$ and $\bm{\theta}$ in time step $t$ follows the alternating manner as
\begin{equation}
	\begin{aligned}
		{\bm{\theta}}_{t + 1} &= \bm{\theta}_{t} - \eta_{\bm{\theta}}\nabla_{\bm{\theta}}\mathcal{L}_{\rm CE} \left(\bm{\theta}_{t},\bm{\alpha}_{t}; \mathcal{D}_{\rm train}\right), \\
		{\bm{\alpha}}_{t + 1} &= \bm{\alpha}_{t} - \eta_{\bm{\alpha}}\nabla_{\bm{\alpha}}\mathcal{L}_{\rm CE}\left(\bm{\theta}_{t + 1},\bm{\alpha}_{t};\mathcal{D}_{\rm valid}\right),
	\end{aligned}
\end{equation}
where 
\begin{align}
	\mathcal{L}_{\rm CE} \left( \bm{\alpha}, \bm{\theta}; \mathcal{D} \right) = \mathbb{E}_{(x,y) \in \mathcal{D}} \left[ \textbf{y}^T \log p_{\bm{\alpha}, \bm{\theta}} \left( \mathbf{y}| \mathbf{x} \right) \right]
	\label{eq:ce}
\end{align}
is the cross-entropy (CE) loss.
\subsection{Topology-Constrained Super-network Optimization and Discretization}
In DARTS-like convolution neural architectures, there are two types of cells, namely the normal cell and the reduction cell, respectively. Totally, there are two reduction cells locating at the 1/3 and 2/3 network depths, respectively, whose function is to reduce the image resolutions by half and double the number of channel maps. To stabilize the optimization, it is always assumed that cells of the same type share the same topological structure in DARTS-like approaches. Correspondingly, there are only two distinct architectural sets $\bm{\alpha} = \{ \bm{\alpha}_{\rm normal}, \bm{\alpha}_{\rm reduce} \}$ in the super-network $f_{(\bm{\alpha},\bm{\theta})}^{\rm Super} \left( \textbf{x} \right)$. 

After the optimization of the super-network, a discretization mechanism is required to determine the searched neural architecture. DARTS-like approaches assume that a computing node should receive information from two distinct nodes. To determine the preserved operators. The discretization mechanism follows two steps. First, for edge $(i,j)$ where node $j$ is a computing node, the strongest operator $o_{(i,j)}^{\rm MAX}$ is identified by evaluating the largest contribution weight as
\begin{align}
	o_{(i,j)}^{\rm MAX} = {\arg \max}_{o \in \mathcal{O}} a_{(i,j)}^{o}.
\end{align}
Note that we replace $f_{(i,j)}^{o}$ by $o_{(i,j)}$ for notation simplicity. We also denote the contribution weight for $o_{(i,j)}^{\rm MAX}$ as $a_{(i,j)}^{\rm MAX}$. After then, two strongest operators from previous nodes $\{ o_{(i,j)}^{\rm MAX} \mid i \in \mathcal{N}(j) \}$ are preserved based on the two largest contributions weights in $\{ a_{(i,j)}^{\rm MAX} \mid i \in \mathcal{N}(j) \}$.

\section{Flexible Differentiable Architecture Search}
\subsection{Breaking the Topological Constraints in DARTS}
It is worth emphasizing that while the topology-constrained strategy contributes to stabilizing the optimization process of NAS and improving the transferability of derived architectures, it inherently limits the flexibility and potential representational capacity of neural architectures. Assuming that architectures designed within a constrained search space remain optimal in an unconstrained search space is an unreasonable premise, as depicted in Fig. \ref{fig:nascomparison}. Such constraints significantly reduce the diversity of neural architectures with varying computational complexities. These restrictive rules, to some extent, conflict with the fundamental objectives of NAS, which aims to explore a broad and diverse range of architectural possibilities to achieve optimal performance.

Since different operators exhibit varied preferences in feature extraction, there may be advantages in processing features at various depths using neural cells with diverse structures. For example, $3\times3$ and $5\times5$ convolution operators tend to extract texture features due to their limited receptive fields, whereas $3\times3$ and $5\times5$ dilated convolution operators capture broader features because of their expanded receptive fields.

The above analysis motivates us to explore the untapped potential of NAS by the exploration of more flexible architectures. In particular, two critical modifications are proposed as
\begin{itemize}[]
	\item The topology-sharing constraints of cells in DARTS is entirely removed to permit unique structures for each cell. This modification results in an exponentially enlarged search space which is embodied in the cell-unique architectural parameters ${\bm \alpha} = \{ {\bm \alpha}_{1}, {\bm \alpha}_{2},...,{\bm \alpha}_{L} \}$.
	
	\item The discretization mechanism, which restricts each computing node to accept two distinct inputs, is eliminated to further reduce the manual priors on architectures.
\end{itemize}
To adapt to these changes, the information propagation within a cell is modified accordingly. Specifically, the architectural parameters of node $j$ are normalized in a node-wise manner, which results in the contribution weights of operators in FX-DARTS being calculated as
\begin{align}
	a_{(i,j)}^{o} = \frac{ \exp\left(\alpha_{(i,j)}^{o} \right) }{ \sum_{i' \in\mathcal{N}(j)} \sum_{ o' \in \mathcal{O}}{\exp \left(\alpha_{(i,j)}^{o'} \right)}}.
	\label{eq:contribution_weights}
\end{align}
Correspondingly, the value of node $j$ is reformulated as
\begin{align}
	x_{j} = \sum_{i \in \mathcal{N}(j)} \sum_{o\in \mathcal{O}} a_{(i,j)}^{o} f_{(i,j)}^{o}\left(x_{i}\right).
\end{align}

%
%

\subsection{Flexible DARTS}
The absence of topological constraints in search spaces introduces fundamental challenges for neural architecture search, i.e., how to determine which operators to retain when no prior structural rules exist, and how to balance exploration of diverse architectures with exploitation of optimal configurations. To address these challenges, we propose FX-DARTS, a novel differentiable architecture search method enhanced with an Entropy-based Super-network Shrinking (ESS) framework, as detailed in Algorithm \ref{alg:ESS}. Below presents the key component systematically.
\begin{algorithm}[h]
	\caption{Entropy-based Super-network Shrinking (ESS)}
	\label{alg:ESS}
	\textbf{Input: }{Dataset $\mathcal{D}_{\rm train}$, shrinking coefficients $c_{1}$ and $c_{2}$, reinitialization rounds $R_{\rm init}$ and architecture optimization epoch $T_{\rm search}$}. \\
	\textbf{Begin: }
	\begin{algorithmic}[1]
		\STATE  Initialize the warm-up epoch $T_{\rm warm}$ by $T_{\rm warm} \leftarrow T_{\rm search} / 2$. \\
		\item[] \textcolor{gray}{\textbf{// Initialization of the super-network	$f_{(\bm{\alpha},\bm{\theta})}^{\rm Super}(\textbf{x})$}} 
		\STATE Let $\bm{\alpha} = \{ \}$.
		\FOR {each cell $k \in [1,...,L]$, each computing node $j \in [3,...,N-1]$, each input node $i \in \mathcal{N}(j)$, and each operator $o \in \mathcal{O}$}
		\STATE Initialize $\alpha_{(k,i,j)}^{o} \leftarrow 0$ and append it to $\bm{\alpha}$.
		\ENDFOR
		\STATE Initialize the expected entropy reduction $\Delta E$ by \eqref{eq:delta_entropy}. \\
		\STATE{Architectures $\mathcal{A} = \{ \}$}.
		\FOR{$R \in \left[1,...,R_{\rm init}\right]$ }
		\item[] \textcolor{gray}{\textbf{// Cycling reinitialization for the model parameters}}
		\STATE{Initialize the model parameters $\bm{\theta}$ of the super-network.}
		\item[]  \textcolor{gray}{\textbf{// Warm-up phase}}
		\FOR{$T \in [1, ..., T_{\rm warm}]$ and each step}
		\STATE Compute the CE loss $\mathcal{L}\left( \bm{\theta}, \bm{\alpha} \right)$ by \eqref{eq:ce}.
		\STATE Update $\bm{\theta} \leftarrow \bm{\theta} - \eta_{\bm{\theta}}\nabla_{\bm{\theta}}\mathcal{L}(\bm{\theta},\bm{\alpha}; \mathcal{D}_{\rm train})$.
		\STATE Update $\bm{\alpha} \leftarrow \bm{\alpha} - \eta_{\bm{\alpha}}\nabla_{\bm{\alpha}}\mathcal{L}(\bm{\theta},\bm{\alpha};\mathcal{D}_{\rm train})$.
		\ENDFOR
		
		\item[]  \textcolor{gray}{\textbf{// Architecture optimization phase}}
		\FOR{$T \in [T_{\rm warm} + 1, ..., T_{\rm search}]$ and each step}
		\STATE Compute the overall loss $\mathcal{L}_{\rm All}\left( \bm{\theta}, \bm{\alpha} ; \lambda_{1:L} \right)$ by \eqref{eq:overall_loss}.
		\STATE Update $\bm{\alpha} \leftarrow \bm{\alpha} - \eta_{\bm{\alpha}}\nabla_{\bm{\alpha}}\mathcal{L}_{\rm All}\left(\bm{\theta},\bm{\alpha}; \lambda_{1:L} \right)$.
		\STATE{Execute Algorithm \ref{alg:discretization} to discretize the super-network}.
		\item[]  \textcolor{gray}{\textbf{// Feedback-based adaptive coefficient adjustment}}
		\FOR{$k \in [1,...,L]$}
		\STATE{Calculate cell entropy $E_{k}$ using \eqref{eq:cell_entropy}.}
		\IF{$E_{t-1}^{k} - E_{t}^{k} < \Delta E$}
		\STATE{$\lambda_{k} \leftarrow c_{1}\lambda_{k}$ where $c_{1} > 1$.} 
		\ELSE
		\STATE{$\lambda_{k} \leftarrow c_{2} \lambda_{k}$ where $0 < c_{2} < 1$.}
		\ENDIF
		\ENDFOR
		\ENDFOR
		\item[]  \textcolor{gray}{\textbf{// Save the architecture}}
		\STATE Save current $\bm{\alpha}$ to $\mathcal{A}$.
		\ENDFOR
	\end{algorithmic}
	\textbf{Output: }$\mathcal{A}$.
\end{algorithm}
\paragraph{Information Entropy Loss} Let $\bm{\alpha} = \{\alpha_{(k,i,j)}^o\}$ denote architecture parameters where $\alpha_{(k,i,j)}^o$ represents the meta-parameter for selecting operator $o$ between nodes $(i,j)$ in cell $k$, and let $\bm{\theta}$ be the model parameters of the super-network. To guide the super-network toward a sparse structure (the shrinking process), we introduce an information entropy loss to measure the sparsity of the super-network at each computing node. Specifically, for node \( j \) in cell $k$, the entropy loss is defined as

\begin{align}  
	\mathcal{H}^{\rm node}_{(k,j)}\left( \bm{\alpha} \right) =  - \sum_{i \in \mathcal{N}(j)} \sum_{o\in \mathcal{O}} a_{(k,i,j)}^{o} \cdot \log a_{(k,i,j)}^{o},
\end{align}  
where \( a_{(k,i,j)}^{o} \) represents the contribution weight of operator \( o \) between nodes \( i \) and \( j \) of cell $k$. Specifically, minimizing this term encourages the contribution weights to concentrate on fewer operators, thereby the sparsity of the super-network is induced. By summing the entropy losses for all computing nodes, we obtain the cell-level sparsity entropy as

\begin{align}  
	\mathcal{H}_{(k)}^{\rm cell} \left( \bm{\alpha} \right) = \sum_{2 <  j < N} \mathcal{H}^{\rm node}_{(k,j)}\left( \bm{\alpha} \right).  
	\label{eq:cell_entropy}  
\end{align}  
Minimizing \( \mathcal{H}_{(k)}^{\rm cell} \left( \bm{\alpha} \right) \) encourages the cell $k$ toward a sparser structure, thereby the derivation of discrete architectures of the super-network is simplified. To balance the model performance and sparsity, an overall loss function for the super-network is defined as
\begin{align}  
	\mathcal{L}_{\rm All}\left( \bm{\theta}, \bm{\alpha} \right) = \mathcal{L}_{\rm CE}\left( \bm{\theta}, \bm{\alpha};\mathcal{D}_{\rm train} \right) + \sum_{k=1}^{L} \lambda_{k} \mathcal{H}_{(k)}^{\rm cell} \left( \bm{\alpha} \right) ,  
	\label{eq:overall_loss}  
\end{align}  
where \( \lambda_{k} \ (k=1,...,N) \) are  balance coefficients which control the trade-off between model performance \( \mathcal{L}_{\rm CE} \left( \bm{\theta}, \bm{\alpha};\mathcal{D}_{\rm train} \right) \) and sparsity entropy \( \sum_{k=1}^{L} \lambda_{k} \mathcal{H}_{(k)}^{\rm cell} \left( \bm{\alpha} \right) \) of all cells. Notably, each \( \lambda_{k} \) is an adaptive hyper-parameter adjusted based on the difference between the expected and actual shrinking speeds of the super-network as discussed below.  

\paragraph{Feedback-based Adaptive Coefficient Adjustment} The ESS framework incorporates a feedback mechanism inspired by control theory to dynamically adjust \( \lambda_{k} \). Specifically, at each training step \( t \), we compute the sparsity entropy \( E^{t}_{k} = \mathcal{H}^{\rm cell}_{(k)}(\bm{\alpha}) \) for the \( k \)th cell. If the actual entropy reduction \( E_{t-1}^{k} - E_{t}^{k} \) is less than the expected reduction \( \Delta E \), it indicates that the regularization effect of the sparsity entropy is insufficient. In this case, the adaptive parameter $\lambda_{k}$ is enlarged by $\lambda_{k} \leftarrow c_{1} \lambda_{k}$ with $c_{1} > 1$ (e.g., $c_{1} = 1.05$). Otherwise, $\lambda_{k}$ is decreased by $\lambda_{k} \leftarrow c_{2} \lambda_{k}$ with $0 < c_{2} < 1$ (e.g., $c_{2} = 0.95$). This feedback-based adjustment ensures that the super-network shrinks at a controlled and adaptive rate.  

\paragraph{Super-network Warm-up Phase}
To decouple the co-adaptation of super-network training and architectural search, FX-DARTS employs a bifurcated approach consisting of two sequential stages: 1) a warm-up phase and 2) an architectural optimization phase. The warm-up phase primarily focuses on initializing and progressively training the super-network's model parameters while deliberately excluding the sparsity loss term from the optimization objective. During this preliminary stage, both the network parameters $\bm{\theta}$ and architectural weights $\bm{\alpha}$ are jointly optimized through gradient descent over $T_{\rm warm}$ epochs, with the exclusive objective of minimizing the cross-entropy loss defined in \eqref{eq:ce}. This strategic separation enables the super-network to establish stable feature representations before subsequent architectural optimization.

\paragraph{Architectural Optimization Phase}
After completing the warm-up phase, the architectural parameters \(\bm{\alpha}\) are optimized over \( T_{\rm search} - T_{\rm warm}\) epochs by minimizing the overall loss \eqref{eq:overall_loss}. Specifically, this phase which ensures the algorithm systematically evaluates and prioritizes candidate operations based on their contributions to the network performance, ultimately leads to the sparsity of the super-network.

\paragraph{Expected Entropy Reduction} It is easy to control the shrinking speed of the super-network with the ESS algorithm. When $\Delta E$ is small, more attentions will be paid to the minimization of the CE loss, and it takes a longer time for the super-network to become sparse. When $\Delta E$ is large, the super-network tends to be sparse in a shorter time. In particular, during the training of the super-network, a warm-up phase is applied for the first $T_{\rm warm}$ epochs. During this phase, only the model parameters $\theta$ are updated. Particularly, we intuitively initialize the expected entropy reduction $\Delta E$ as
\begin{align}
	\Delta E \leftarrow \frac{\sum_{k=1}^{L} \mathcal{L}_{\rm Entropy}^{\rm init}(\bm{\alpha_{k}}) }{L{\rm len}({train\_loader}) T_{\rm search} R_{\rm init} },
	\label{eq:delta_entropy}
\end{align}
where $\mathcal{L}_{\rm Entropy}^{\rm init}\left( \bm{\alpha}_{k} \right)$ denotes the initial value of the sparsity entropy of cell $k$ when the optimization begins. Specifically, this formulation establishes an entropy reduction budget per training step that ensures controlled architecture shrinkage. The denominator accounts for: 1) $L$ cells which need independent entropy control; 2) training steps per epoch ${\rm len}(train\_loader)$; 3) total search duration $T_{\rm search} \times R_{\rm init}$. 
\begin{algorithm}[h]
	\caption{The dynamic discretization of the super-network}
	\label{alg:discretization}
	\textbf{Input: }{The super-network $f_{(\bm{\alpha}, \bm{\theta})}\left(\textbf{x}\right)$, and a sufficiently small threshold $\epsilon$ (e.g., $0.02$ in experiments)}. \\
	\textbf{Begin: }	
	\begin{algorithmic}[1]
		\FOR{$k\in [1,...,L], j\in [3,...,N-1], i\in [1,...,j-1]$, and $o\in \mathcal{O}$}
		\STATE{Obtain the contribution weight $a_{(k,i,j)}^{o}$ for operator $o_{(k,i,j)}$ by \eqref{eq:contribution_weights}.}
		\IF{$a_{(k,i,j)}^{o} < \epsilon$}
		\STATE{Delete $o_{(k,i,j)}$.}
		\ENDIF
		\ENDFOR	
	\end{algorithmic}
\end{algorithm}
\paragraph{Cyclic Parameter Reinitialization} To further mitigate the co-adaptation phenomenon where architecture parameters $\bm{\alpha}$ and model parameters $\bm{\theta}$ are optimized jointly or alternately, FX-DARTS introduces cyclic parameter reinitialization to decouple their optimization dynamics. As delineated in Algorithm \ref{alg:ESS}, this process involves $R_{\rm init}$ independent search rounds where $\bm{\theta}$ is reinitialized at each round while $\bm{\alpha}$ accumulates architectural knowledge across rounds. 

\paragraph{Discretization of the Super-network} As the super-network evolves towards a sparser structure, operators with contribution weights below a predetermined threshold are systematically pruned. This pruning strategy ensures that the removal of such operators has minimal impact on the overall performance of the super-network. The discretization process for the super-network is outlined in Algorithm \ref{alg:discretization}. Notably, this discretization approach does not enforce rigid prior constraints on the resulting discrete architectures. Consequently, a diverse set of architectures, characterized by flexible structures and varying computational complexities, can be efficiently derived within a single search procedure.

\subsection{Convergence analysis of the sparsity entropy}
In this part, we theoretically show that under some assumptions, the monotonic decrement of $\mathcal{H}_{(k)}$ can be achieved for each optimization step whenever $\mathcal{H}_{(k)} > 0$. To simplify the discussion, we assume there is only one neural cell in the super-network, such that $\mathcal{H}^{\rm cell} \equiv \mathcal{H}_{(k)}^{\rm cell}$, $\lambda \equiv \lambda_k$ and $\mathcal{H}^{\rm node}_{(j)} \equiv \mathcal{H}^{\rm node}_{(k,j)}$. Correspondingly, the optimization objective is $\mathcal{L}_{\rm All} = \mathcal{L}_{\rm ce} + \lambda \mathcal{H}^{\rm cell}$.

\begin{remark}
	The node-wise sparsity entropy $\mathcal{H}^{\rm node}_{(j)}$ exhibits the following key properties:
	\begin{enumerate}
		\item Non-negativity: $\mathcal{H}^{\rm node}_{(j)} \geq 0$.
		\item Minimum value: When the node-wise distribution $P_{j} = \{ a_{(i,j)} \mid i \in \mathcal{N}(j), o \in \mathcal{O} \}$ is one-hot, $\mathcal{H}^{\rm node}_{(j)}$ attains its minimum value of $0$.
		\item Maximum value: When $P_{j}$ is a uniform distribution, $\mathcal{H}^{\rm node}_{(j)}$ reaches its maximum value of $\log \left( |\mathcal{N}(j)| \cdot |\mathcal{O}| \right)$.
	\end{enumerate}
\end{remark}

\begin{proof}
	\
	\begin{enumerate}
		\item 
		Since $a_{(i,j)}^o \in [0,1]$ and $-\log a_{(i,j)}^o \geq 0$ for $a_{(i,j)}^o \in (0,1]$, it follows that $\mathcal{H}^{\rm node}_{(j)} \geq 0$.
		\item 
		Suppose $\mathcal{H}^{\rm node}_{(j)} = 0$. It occurs if and only if one term $a_{(i,j)}^o = 1$ (others are $0$), i.e., $P_j$ is one-hot. Conversely, if $P_j$ is one-hot, then $\mathcal{H}^{\rm node}_{(j)} = 0$.
		\item 
		The maximum entropy is achieved when $P_j$ is a uniform distribution. To see this, let $M = |\mathcal{N}(j)| \cdot |\mathcal{O}|$. Using Lagrange multipliers, we have
		\[
		\max_{\{a_{(i,j)}^o\}} \left( -\sum_{i,o} a_{(i,j)}^o \log a_{(i,j)}^o \right) \quad \text{s.t.} \quad \sum_{i,o} a_{(i,j)}^o = 1.
		\]
		The Lagrangian is
		\[
		\mathcal{L} = -\sum_{i,o} a_{(i,j)}^o \log a_{(i,j)}^o + \lambda \left( \sum_{i,o} a_{(i,j)}^o - 1 \right).
		\]
		By setting $\frac{\partial \mathcal{L}}{\partial a_{(i,j)}^o} = -\log a_{(i,j)}^o - 1 + \lambda = 0$, we get $a_{(i,j)}^o = e^{\lambda-1}$. Since $\sum_{i} a_{(i,j)} = 1$, we have $a_{(i,j)}^o = 1/M$. By substituting it back to $\mathcal{H}_{(j)}^{\rm node}$, we derive
		\[
		\mathcal{H}^{\rm node}_{(j)} = -M \cdot \frac{1}{M} \log \frac{1}{M} = \log M.
		\]
	\end{enumerate}
\end{proof}

\begin{corollary}
	The gradient $\nabla_{\alpha} \mathcal{H}_{(j)}^{\rm node} = \left[ \frac{\partial \mathcal{H}_{(j)}^{\rm node}}{\partial \alpha_{(i,j)}^o} \right]$ satisfies
	\begin{align}
		\frac{\partial \mathcal{H}_{(j)}^{\rm node}}{\partial \alpha_{(i,j)}^o} = -a_{(i,j)}^o \left( \log a_{(i,j)}^o + \mathcal{H}^{\rm node}_{(j)}(\alpha) \right).
	\end{align}
	Specifically, when $P_{j}$ is not one-hot, $\nabla_{\alpha} \mathcal{H}_{(j)}^{\rm node}$ is non-zero. This implies that $\nabla_{\alpha} \mathcal{H}_{(j)}^{\rm node} = 0$ is achieved when $\mathcal{H}_{(j)}^{\rm node} = 0$ only. Furthermore, if $P_{j}$ for $2 < j < N$ are not all one-hot, then $\nabla_{\alpha} \mathcal{H}^{\rm cell}$ is non-zero, which ensures that $\|\nabla_{\alpha} \mathcal{H}^{\rm cell}\| > 0$.
	\label{remark:2}
\end{corollary}
\begin{proof}
	The contribution weights $a_{(i,j)}^o$ are derived through
	\[
	a_{(i,j)}^o = \frac{\exp(\alpha_{(i,j)}^o)}{\sum_{i',o'} \exp(\alpha_{(i',j)}^{o'})}.
	\]
	The gradient of $\mathcal{H}^{\rm node}_{(j)}$ with respect to $\alpha_{(i,j)}^o$ is
	\[
	\frac{\partial \mathcal{H}^{\rm node}_{(j)}}{\partial \alpha_{(i,j)}^o} = -\left( \log a_{(i,j)}^o + 1 \right) \cdot a_{(i,j)}^o + \sum_{i',o'} a_{(i',j)}^{o'} \log a_{(i',j)}^{o'} \cdot a_{(i,j)}^o.
	\]
	By simplifying the above formulation with $\sum_{i',o'} a_{(i',j)}^{o'} = 1$, we get
	\[
	\frac{\partial \mathcal{H}^{\rm node}_{(j)}}{\partial \alpha_{(i,j)}^o} = -a_{(i,j)}^o \left( \log a_{(i,j)}^o + \mathcal{H}^{\rm node}_{(j)} \right).
	\]
	If $P_j$ is not one-hot, $\mathcal{H}^{\rm node}_{(j)} > 0$, so $\nabla_{\alpha} \mathcal{H}^{\rm node}_{(j)} \neq 0$. Extending the results to a neural cell,  we have $\nabla_{\alpha} \mathcal{H}^{\rm cell} \neq 0$ if any $P_j$ are not one-hot for $ 2 < j < N$.

\end{proof}
\begin{assumption}[Boundedness of $\nabla_{\bm{\alpha}} \mathcal{L}_{\rm CE}$]
At each optimization step, there exits a constant $g > 0$, such that the gradient of the cross-entropy loss with respect to $\bm{\alpha}$ satisfies $\|\nabla_{\bm{\alpha}} \mathcal{L}_{\rm CE}\| \leq g$, where $\|\cdot\|$ denotes the $L_2$ norm.
\label{assumption}
\end{assumption}
\begin{theorem}
	Under the first-order Taylor approximation, when $\mathcal{H}^{\rm cell} > 0$, with a reasonably small learning rate $\eta$ for updating $\bm \alpha$, there exists a $\lambda > - \frac{\|\nabla_{\bm{\alpha}} \mathcal{L}_{\rm CE}\| \cos \theta}{\|\nabla_{\bm{\alpha}} \mathcal{H}^{\rm cell}\|}$ such that the entropy change $\Delta \mathcal{H}^{\rm cell}$ for arbitrary architectural optimization steps is negative.
	\label{theorem:guarantee}
\end{theorem}

\begin{proof}
	In an optimization step, the updating of $\bm{\alpha}$ is given by
	\begin{align}
		\Delta \bm{\alpha} = -\eta \left( \nabla_{\bm{\alpha}} \mathcal{L}_{\rm CE} + \lambda \nabla_{\bm{\alpha}} \mathcal{H}^{\rm cell} \right).
	\end{align}
	Correspondingly, the variation of $\mathcal{H}^{\rm cell}$ can be approximated using a first-order Taylor expansion as
	\begin{align}
		\Delta & \mathcal{H}^{\rm cell} \approx {\nabla_{\bm{\alpha}} \mathcal{H}^{\rm cell}}^\top \Delta \bm{\alpha} \nonumber \\
		&= -\eta {\nabla_{\bm{\alpha}} \mathcal{H}^{\rm cell}}^\top \nabla_{\bm{\alpha}} \mathcal{L}_{\rm CE} - \eta \lambda \|\nabla_{\bm{\alpha}} \mathcal{H}^{\rm cell}\|^2 \nonumber \\
		&= -\eta \|\nabla_{\bm{\alpha}} \mathcal{H}^{\rm cell}\| \|\nabla_{\bm{\alpha}} \mathcal{L}_{\rm CE}\| \cos \theta - \eta \lambda \|\nabla_{\bm{\alpha}} \mathcal{H}^{\rm cell}\|^2,
		\label{eq:entropy}
	\end{align}
	where $\theta$ is the angle between $\nabla_{\bm{\alpha}} \mathcal{H}^{\rm cell}$ and $\nabla_{\bm{\alpha}} \mathcal{L}_{\rm CE}$. Since $\mathcal{H}^{\rm cell} > 0$, from Remark~\ref{remark:2}, we know that $\|\nabla_{\bm{\alpha}} \mathcal{H}^{\rm cell}\| \neq 0$. Additionally, by Assumption~\ref{assumption}, $\|\nabla_{\bm{\alpha}} \mathcal{L}_{\rm CE}\| \leq g$. Therefore, when
	\begin{align}
		\lambda > - \frac{\|\nabla_{\bm{\alpha}} \mathcal{L}_{\rm CE}\| \cos \theta}{\|\nabla_{\bm{\alpha}} \mathcal{H}^{\rm cell}\|},
	\end{align}
	the term $-\eta \|\nabla_{\bm{\alpha}} \mathcal{H}^{\rm cell}\| \|\nabla_{\bm{\alpha}} \mathcal{L}_{\rm CE}\| \cos \theta - \eta \lambda \|\nabla_{\bm{\alpha}} \mathcal{H}^{\rm cell}\|^2$ is guaranteed to be negative.
\end{proof}

\begin{remark}
Specifically, the first term in \eqref{eq:entropy} represents the performance-driven term, which can become positive when the angle $\theta$ between $\nabla_{\bm{\alpha}} \mathcal{H}^{\rm cell}$ and $\nabla_{\bm{\alpha}} \mathcal{L}_{\rm CE}$ exceeds $\pi/2$. 
\end{remark}

\begin{corollary}
	When $\lambda = \frac{\Delta E - \eta \|\nabla_{\bm{\alpha}} \mathcal{H}^{\rm cell}\| \|\nabla_{\bm{\alpha}} \mathcal{L}_{\rm CE}\| \cos \theta}{\eta \|\nabla_{\bm{\alpha}} \mathcal{H}^{\rm cell}\|^2}$, it is guaranteed that $-\Delta \mathcal{H}^{\rm cell} \approx \Delta E$.
	\label{corollary:converge}
\end{corollary}

\begin{proof}
	To achieve $-\Delta \mathcal{H}^{\rm cell} \approx \Delta E$,
	let 
	\begin{align}
	-\Delta \mathcal{H}^{\rm cell} &= \eta \|\nabla_{\bm{\alpha}} \mathcal{H}^{\rm cell}\| \|\nabla_{\bm{\alpha}} \mathcal{L}_{\rm CE}\| \cos \theta + \eta \lambda  \|\nabla_{\bm{\alpha}} \mathcal{H}^{\rm cell}\|^2  
	\nonumber
	\\ & = \Delta E.
	\end{align}
	By solving this equation for $\lambda$, we obtain
	\[
	\lambda = \frac{\Delta E - \eta \|\nabla_{\bm{\alpha}} \mathcal{H}^{\rm cell}\| \|\nabla_{\bm{\alpha}} \mathcal{L}_{\rm CE}\| \cos \theta}{\eta \|\nabla_{\bm{\alpha}} \mathcal{H}^{\rm cell}\|^2}.
	\]
\end{proof}

\begin{remark}
	Given a predefined minimum entropy $\mathcal{H}_{\rm min}^{\rm cell} > 0$ during the search, by using a sufficiently small learning rate $\eta$ for updating $\bm{\alpha}$ and a small expected entropy reduction $\Delta E > 0$, the sparsity entropy $\mathcal{H}^{\rm cell}$ can be monotonically reduced to $\mathcal{H}_{\rm min}^{\rm cell}$ by appropriately adjusting $\lambda$, as outlined in Theorem~\ref{theorem:guarantee} and Corollary~\ref{corollary:converge}. While the exact value of $\lambda$ can be computed at each optimization step, $\lambda$ is adaptively adjusted by ESS based on the actual entropy reduction $E_{t-1} - E_t$, where $E_{t} = \mathcal{H}^{\rm cell}\left({\bm{\alpha}}\right)$ at step $t$, and the target decrement $\Delta E$. The adaptive adjustment allows more flexibility in entropy-based regularization. In expectation, the feedback-based adaptive adjustment ensures that $\lambda$ fluctuates around its exact value.
\end{remark}


\begin{figure*}[htb]
	\subfloat[\tiny FX-DARTS-$O_{1}$ (48E)]{
		\includegraphics[height=0.29\textheight]{architectures/genotype_large_48epoch_S1}
		\label{fig:genotype_48epoch_S1}
	}
	\subfloat[\tiny FX-DARTS-$O_{1}$ (64E)]{
		\includegraphics[height=0.29\textheight]{architectures/genotype_large_64epoch_S1}
		\label{fig:genotype_64epoch_S1}
	}
	\subfloat[\tiny FX-DARTS-$O_{2}$ (48E)]{
		\includegraphics[height=0.29\textheight]{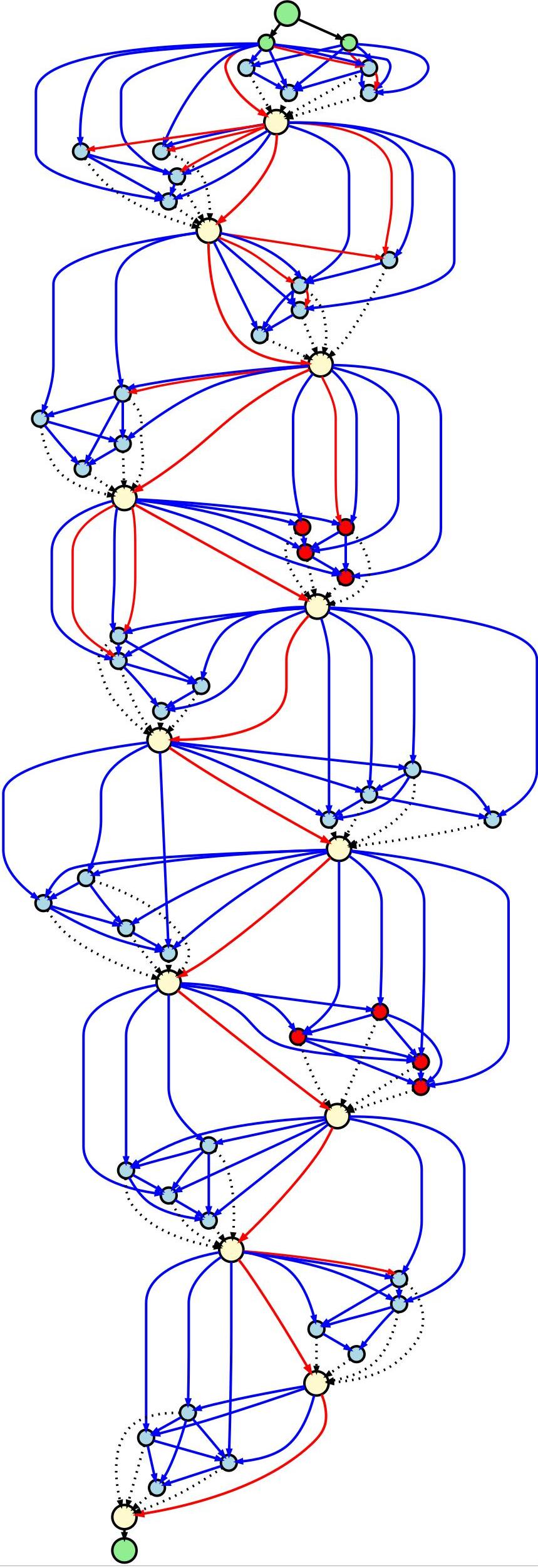}
		\label{fig:genotype_48epoch_S2}
	}
	\subfloat[\tiny FX-DARTS-$O_{2}$ (64E)]{
		\includegraphics[height=0.29\textheight]{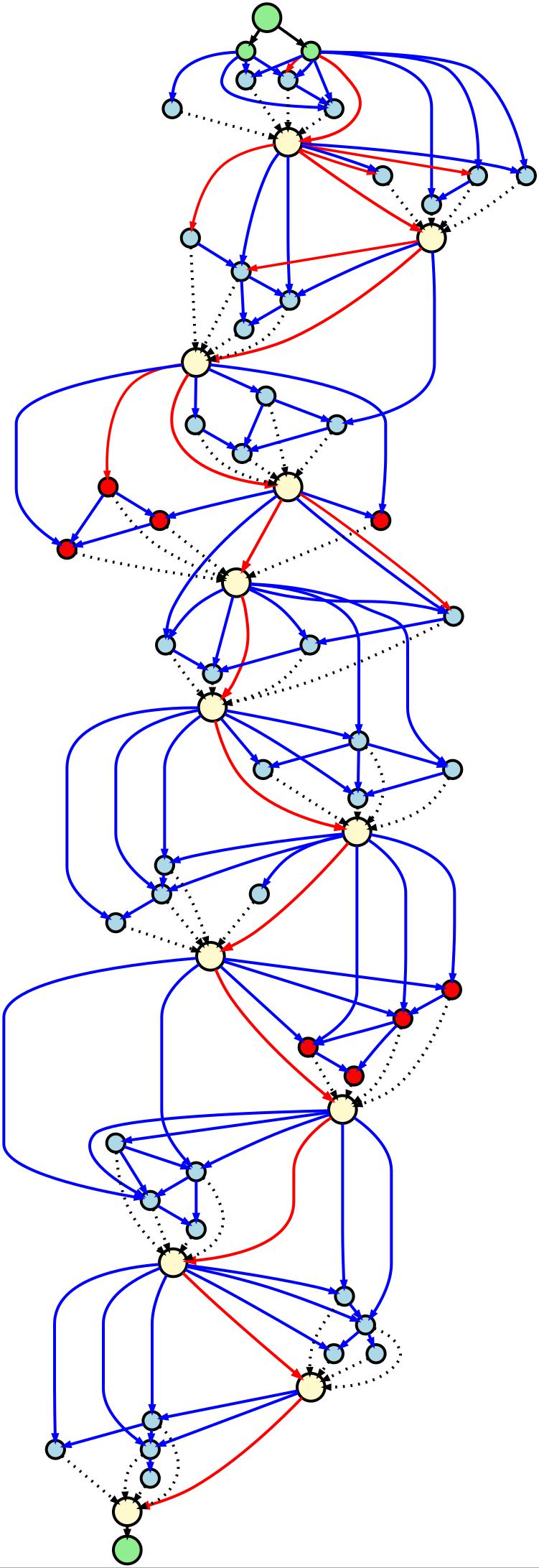}
		\label{fig:genotype_64epoch_S2}
	}
	\subfloat[\tiny FX-DARTS-$O_{3}$ (48E)]{
		\includegraphics[height=0.29\textheight]{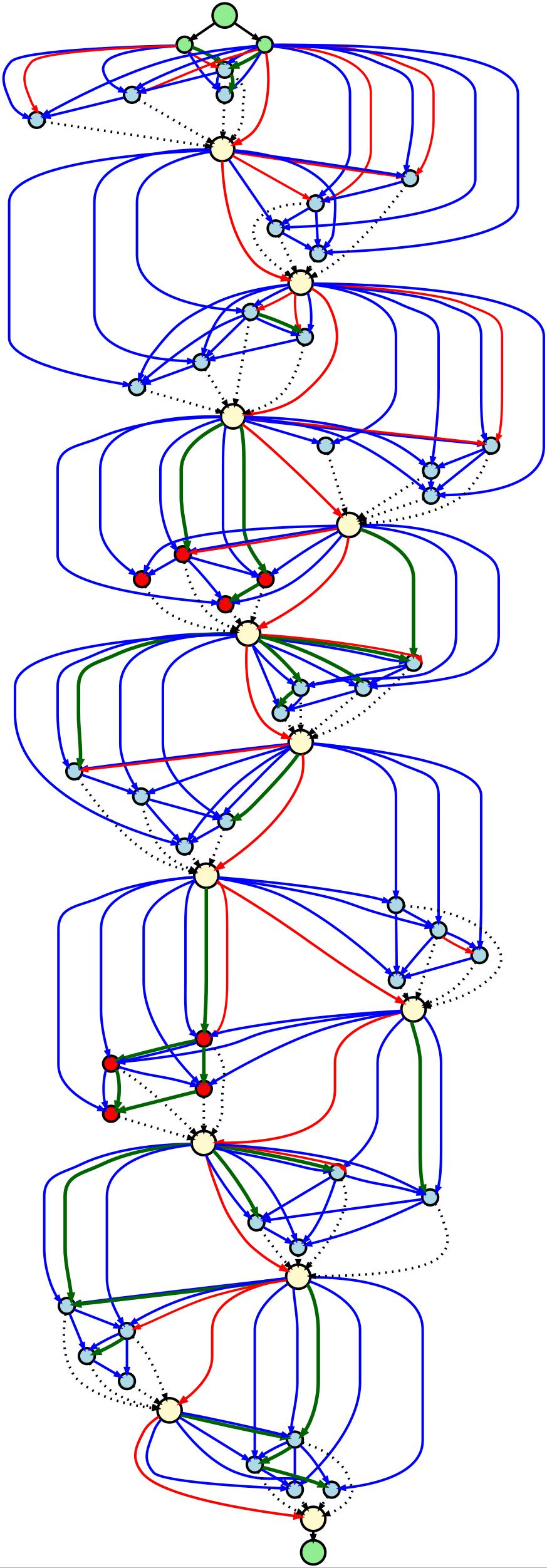}
		\label{fig:genotype_48epoch_S3}
	}
	\subfloat[\tiny FX-DARTS-$O_{3}$ (64E)]{
		\includegraphics[height=0.29\textheight]{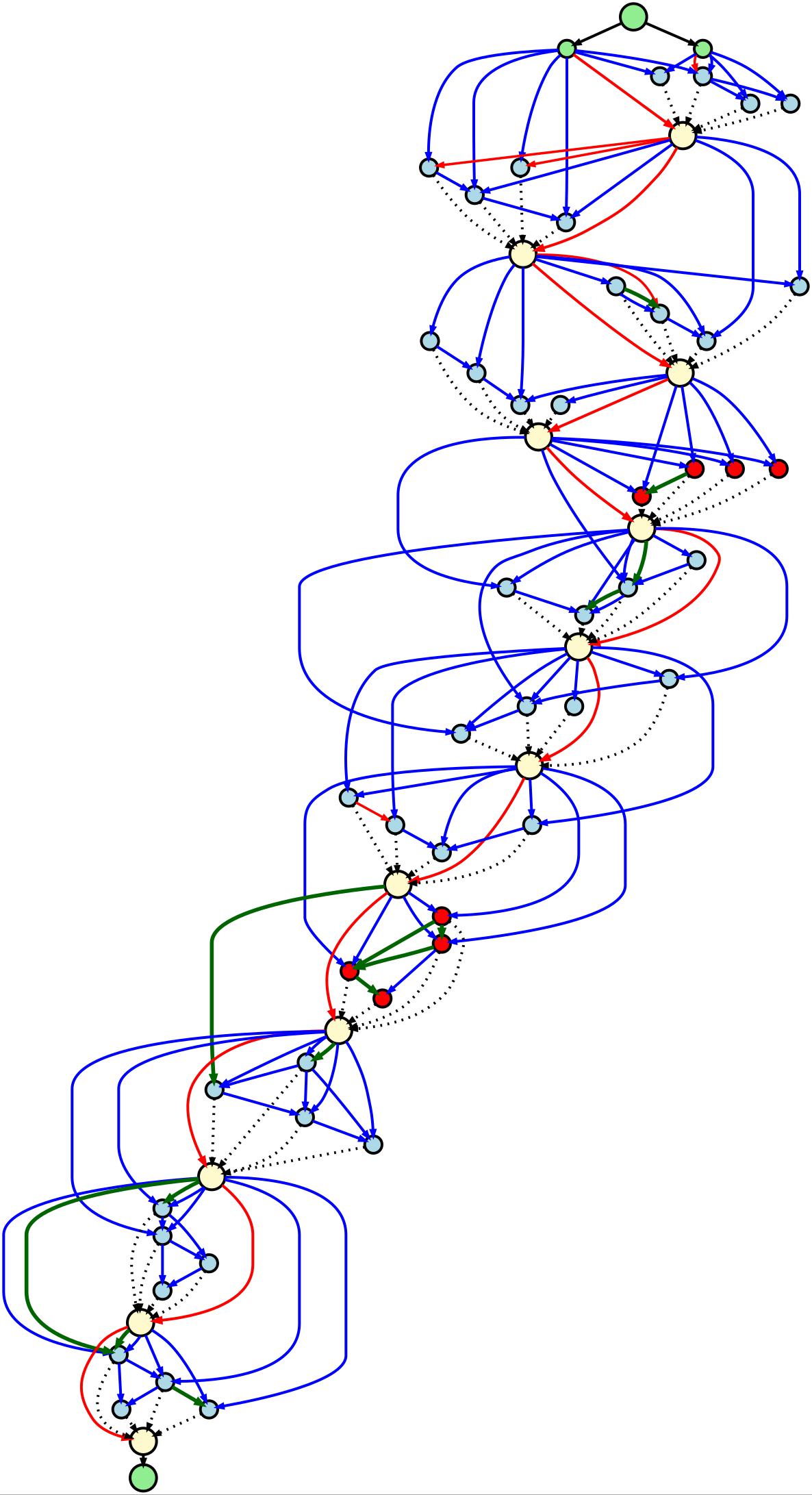}
		\label{fig:genotype_64epoch_S3}
	}
	
	\caption{The networks searched by FX-DARTS on CIFAR-10. \textcolor{lightlightgreen}{Green} nodes are the input, pre-processed tensors and the output.  \textcolor{lightblue}{Light-blue} nodes and \textcolor{darkred}{dark-red} nodes are the intermediate nodes for normal cells and reduction cells, respectively.
		\textcolor{gray}{Grey} dashed lines indicate the concatenation of feature tensors at channel dimension. In particular, \textcolor{red}{red} directed lines are skip-connections, \textcolor{blue}{blue} directed lines are $3\times3$ depthwise separable convolutions, 
		and \textcolor{lightgreen}{green} directed lines are $3\times3$ dilated convolutions.
	}
	\label{fig:network_cifar}
\end{figure*}

\subsection{Architecture search on CIFAR-10/100}
In this section, we present the implementation details of the proposed FX-DARTS on CIFAR-10 classification and conduct comprehensive comparisons with popular NAS architectures on both CIFAR-10 and CIFAR-100 classification tasks. 
\subsubsection{Hyper-parameter settings}\paragraph{Architecture search} For architecture search, we construct the training dataset $\mathcal{D}_{\rm train}$ by randomly selecting 90\% of samples from the CIFAR-10 training set, which is processed in mini-batches of 128 samples. The data augmentation pipeline includes the following transformations: Randomly cropping and resizing of $32\times 32$ with the area ranging from 20\% to 100\% of the original, horizontal flipping with the probability of 50\%, color adjusting (brightness, contrast, saturation, and hue) with the probability of 80\%, grayscale transformation with the probability of 20\%, and CIFAR-10 normalization. We try three different operator spaces, which are
\begin{itemize}
	\item $\mathcal{O}_{1}$: {skip-connect};
	\item $\mathcal{O}_{2}$: {skip-connect, $3\times3$ sep-conv};
	\item $\mathcal{O}_{3}$: {skip-connect, $3\times3$ sep-conv, $5\times5$ dil-conv}.
\end{itemize}
Particularly, skip-connect means the residual connection, sep-conv means the depth-wise separable convolution, and dil-conv means the dilated convolution. The parameters of these operators are initialized by PyTorch's default settings. The super-network consists of $L=12$ cells, and each cell contains $N = 6$ nodes. The super-network architecture comprises $L=12$ cells, with each cell containing $N=6$ nodes. Architectural parameters are initialized to zero, and we employ shrinking coefficients $c_{1}=1.05$ and $c_{2}=0.95$. The initial value of $\lambda_{k}$ is set to $10^{-4}$ for all $k \in [1,...,L]$. The search process is configured with $T_{\rm search}=16$ epochs and $T_{\rm warm}={T_{\rm search}}/{2}$ warm-up epochs. With $R_{\rm init}=5$ rounds of model parameter reinitialization, the total NAS procedure requires $R_{\rm init} \times T_{\rm search}$ epochs. For optimization, the model parameters $\bm{\theta}$ are optimized by the Adam optimizer with a constant learning rate of $10^{-3}$ and a weight decay of $10^{-4}$. The architectural parameters $\bm{\alpha}$ are also optimized by the Adam optimizer, but with a learning rate of $10^{-2}$ and zero weight decay. The expected entropy $\Delta E$ is calculated by \eqref{eq:delta_entropy}, and the discretization threshold $\epsilon$ in Algorithm \ref{alg:discretization} is 0.02.   

\paragraph{Architecture evaluation} As highlighted in prior research, the evaluation outcomes of architectures derived from the DARTS search space are significantly affected by various hyperparameters, including data augmentation techniques \cite{DBLP:conf/iclr/YangEC20}. This observation has prompted us to eliminate the use of drop-path, a technique commonly employed in DARTS-like methodologies, to ensure that the evaluation results more accurately reflect the intrinsic qualities of the architectures themselves. Specifically, we have re-evaluated the architectures reported in prominent NAS literature under uniform training conditions to facilitate equitable comparisons. Owing to the adaptive nature of the ESS framework, FX-DARTS retrains a greater number of operators within discrete architectures under identical cell number configurations. Consequently, we have adjusted the number of cells to 17 for other architectures to maintain parity in comparisons. For both CIFAR-100 and CIFAR-10 evaluations, the initial number of network channels is established to be 20, with a training duration of 150 epochs. The optimization of model parameters is conducted using an SGD optimizer, characterized by an initial learning rate of 0.05, a weight decay of \(3 \times 10^{-4}\), gradient clipping at 5.0, a batch size of 128, and a momentum of 0.9. The learning rate is gradually reduced to zero following a cosine scheduling strategy throughout the training phase. Image augmentation is performed using cutout with dimensions of \(16 \times 16\). To mitigate the impact of randomness, each architecture is evaluated three times. Means and the standard deviations are reported. 

\subsubsection{Results on CIFAR-10/100}
\begin{table*}[htb]
	\centering
	\caption{Comparison results of NAS architectures on CIFAR-10/100 classification}
	\scalebox{0.9}{
	\begin{tabular}{lccccccccccc}
		\toprule     
		{Architecture} & Params (M) & FLOPs (M) & \multicolumn{4}{c}{CIFAR-100 Test Acc. (\%)} &  \multicolumn{4}{c}{CIFAR-10 Test Acc. (\%)} & Search Cost \\
		\cmidrule(lr){4-7} \cmidrule(lr){8-11}
		 &  &  & \# 1 & \# 2 & \# 3 & Avg. & \# 1 & \# 2 & \# 3 & Avg. & (GPU-Time / GPU-Mem) \\
		\midrule
		ResNet-32 \cite{DBLP:conf/cvpr/HeZRS16} & 0.46 & 69 & 71.56 & 72.13 & 72.01 & 71.90$\pm$0.30 & 94.18 & 94.11 & 94.00 & 94.10$\pm$0.09 & - \\
		ResNet-56 \cite{DBLP:conf/cvpr/HeZRS16} & 0.85 & 127 & 73.98 & 74.19 & 74.12 & 74.10$\pm$0.11 & 95.14 & 95.10 & 94.98 & 95.07$\pm$0.08 & - \\
		ResNet-110 \cite{DBLP:conf/cvpr/HeZRS16} & 1.73 & 255 & 75.42 & 75.99 & 75.25 & 75.55$\pm$0.39 & 95.51 & 95.60 & 95.54 & 95.55$\pm$0.05 & - \\
		\midrule
		DARTS-$\rm 1st$ \cite{liu2018darts} & 0.85 & 144 & 76.40 & 76.33 & 75.63 & 76.12$\pm$0.43 & 95.60 & 95.61 & 95.71 & 95.64$\pm$0.06 & 1.5 GPU-days / -  \\
		DARTS-$\rm 2nd$ \cite{liu2018darts} &  0.90 & 152 & 76.87 & 76.64 & 76.45 & 76.65$\pm$0.21 & 95.76 & 95.54 & 95.80 & 95.70$\pm$0.14 & 4.0 GPU-days / - \\
		PC-DARTS \cite{xu2019pc} &  1.02 & 165 & 76.68 & 76.68 & 76.49 & 76.62$\pm$0.11 & 95.85 & 95.91 & 95.99 & 95.92$\pm$0.07 & 0.1 GPU-days / - \\
		P-DARTS \cite{chen2021progressive} & 0.95 & 156 & 76.87 & 77.20 & 77.37 & 77.15$\pm$0.25 & 95.55 & 95.47 & 95.70 & 95.57$\pm$0.12 & 0.1 GPU-days / -\\
		RelativeNAS \cite{tan2021relativenas} & 1.09 & 183 & 76.51 & 76.58 & 76.02 & 76.37$\pm$0.30 & 96.00 & 95.94 & 95.95 & 95.96$\pm$0.03  & 0.4 GPU-days / - \\
		SNAS-mild \cite{xie2018snas} &  0.72 & 120 & 76.04 & 76.21 & 76.70 & 76.32$\pm$0.34 & 95.55 & 95.36 & 95.61 & 95.51$\pm$0.13 & 1.5 GPU-days / - \\
		XNAS \cite{DBLP:conf/nips/NaymanNRFJZ19} & 1.06 & 174 & 77.27 & 77.33 & 77.23 & 77.28$\pm$0.05 & 95.89 & 96.02 & 95.75 & 95.89$\pm$0.14 & 0.3 GPU-days / - \\
		$\beta$-DARTS \cite{DBLP:conf/cvpr/YeL00FO22} & 1.03 & 166 & 76.93 & 76.47 & 77.02 & 76.81$\pm$0.30 & 95.87 & 95.74 & 95.54 & 95.72$\pm$0.17 & 0.4 GPU-days / - \\
		CDARTS \cite{yu2022cyclic} & 1.09 & 173 & 77.43 & 77.12 & 77.02 & 77.19$\pm$0.21 & 95.58 & 95.55 & 96.04 & 95.72$\pm$0.27 & - \\
		Noisy-DARTS \cite{chu2020noisy} & 0.95 & 120 & 76.09 & 76.26 & 76.82 & 76.39$\pm$0.38 & 95.56 & 95.50 & 95.35 & 95.47$\pm$0.11 & 0.4 GPU-days / - \\
		NASNet \cite{baker2016designing} & 1.07 & 180 & 77.33 & 77.36 & 77.45 & 77.38$\pm$0.06 & 95.83 & 95.86 & 96.04 & 95.91$\pm$0.11 & 1800 GPU-days / - \\
		AmoebaNet \cite{DBLP:conf/aaai/RealAHL19} & 0.90 & 147 & 77.29 & 76.52 & 76.92 & 76.91$\pm$0.39 & 95.82 & 95.77 & 95.78 & 95.79$\pm$0.03 & 3150 GPU-days / - \\
		\midrule
		FX-DARTS-$\mathcal{O}_{1}$ (48E) & 1.06 & 165 & 78.46 & 78.20 & 78.31 & 78.32$\pm$0.13 & 95.88 & 95.93 & 96.03 & 95.95$\pm$0.01 & 1.6 GPU-hours / 7.0G \\
		FX-DARTS-$\mathcal{O}_{1}$ (64E) & 0.92 & 148 & 78.42 & 78.17 & 78.02 & 78.20$\pm$0.20 & 96.11 & 95.66 & 95.81 & 95.86$\pm$0.05 & 1.6 GPU-hours / 7.0G \\
		FX-DARTS-$\mathcal{O}_{1}$ (80E) & 0.84 & 133 & 77.92 & 77.86 & 77.90 & 77.89$\pm$0.00 & 95.64 & 95.80 & 95.96 & 95.80$\pm$0.03 & 1.6 GPU-hours / 7.0G \\
		\midrule
		FX-DARTS-$\mathcal{O}_{2}$ (48E) & 1.18 & 181 & 78.09 & 78.20 & 78.44 & 78.24$\pm$0.03 & 96.03 & 95.92 & 96.12 & 96.02$\pm$0.01 & 1.7 GPU-hours / 7.6G \\
		FX-DARTS-$\mathcal{O}_{2}$ (64E) & 0.94 & 142 & 77.63 & 77.25 & 78.20 & 77.69$\pm$0.48 & 95.63 & 95.71 & 95.79 & 95.71$\pm$0.01 & 1.7 GPU-hours / 7.6G\\
		\midrule
		FX-DARTS-$\mathcal{O}_{3}$ (48E) & 1.26 & 195 & 77.97 & 78.18 & 77.63 & 77.93$\pm$0.08 & 95.88 & 96.10 & 95.90 & 95.96$\pm$0.01 & 2.7 GPU-hours / 11G \\
		FX-DARTS-$\mathcal{O}_{3}$ (64E) & 1.01 & 151 & 77.10 & 77.10 & 77.64 & 77.28$\pm$0.31 & 95.99 & 95.85 & 95.88 & 95.91$\pm$0.01 & 2.7 GPU-hours / 11G \\
		\bottomrule
	\end{tabular}
	}
	\begin{tablenotes}
	\scriptsize
	\item The search cost for FX-DARTS is tested on NVIDIA RTX 4090 GPUs.
	\end{tablenotes}
	\label{tab:results_CIFAR-10}
\end{table*}
\begin{figure}
	\centering
	\includegraphics[width=1.0\linewidth]{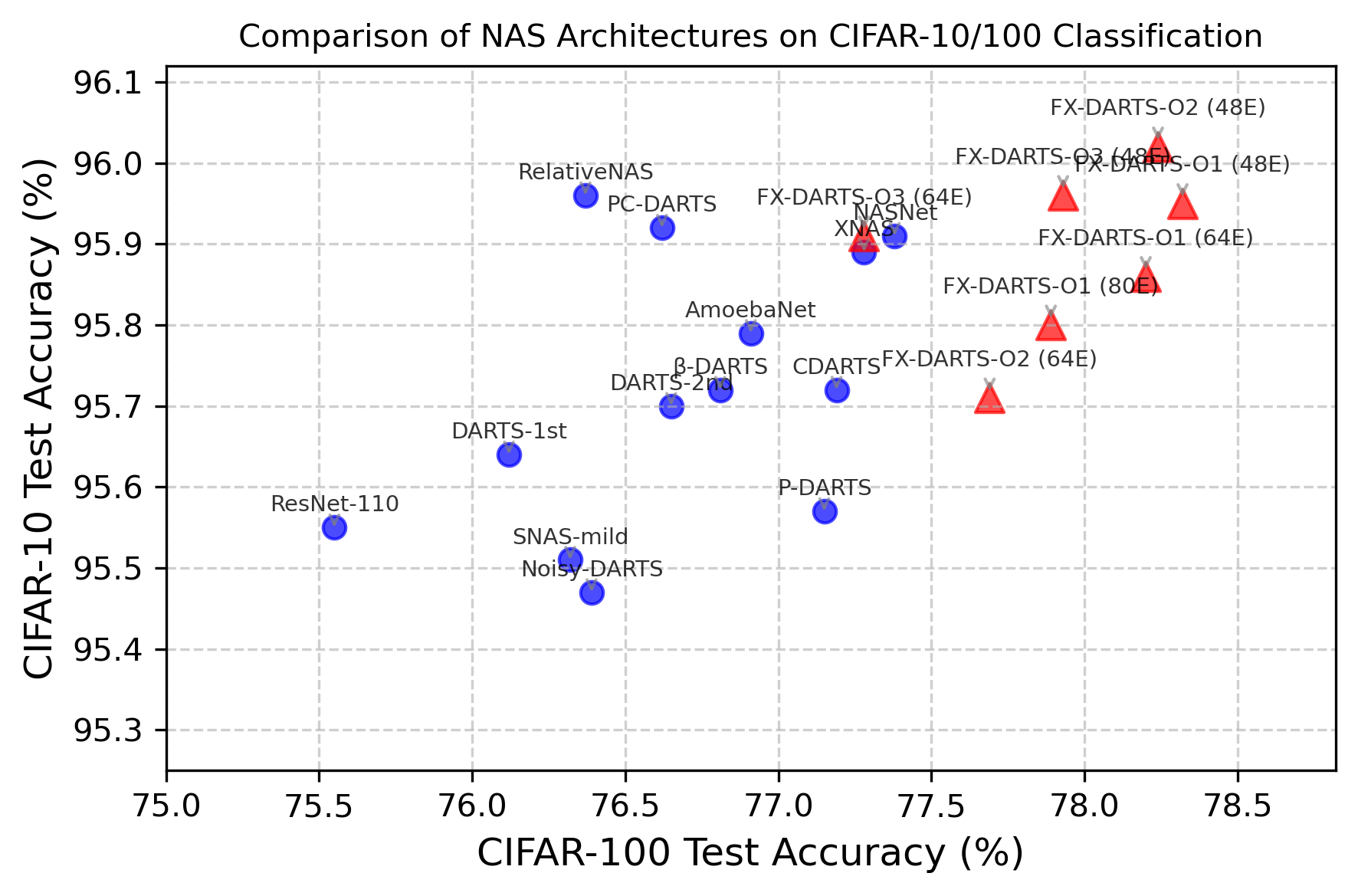}
	\caption{Performance comparison of FX-DARTS vs. Topology-constrained NAS architectures on CIFAR-10/100 classification tasks. The results correspond to Table \ref{tab:results_CIFAR-10}.}
	\label{fig:nascomparison}
\end{figure}

\begin{figure}
	\centering
	\includegraphics[width=1.0\linewidth]{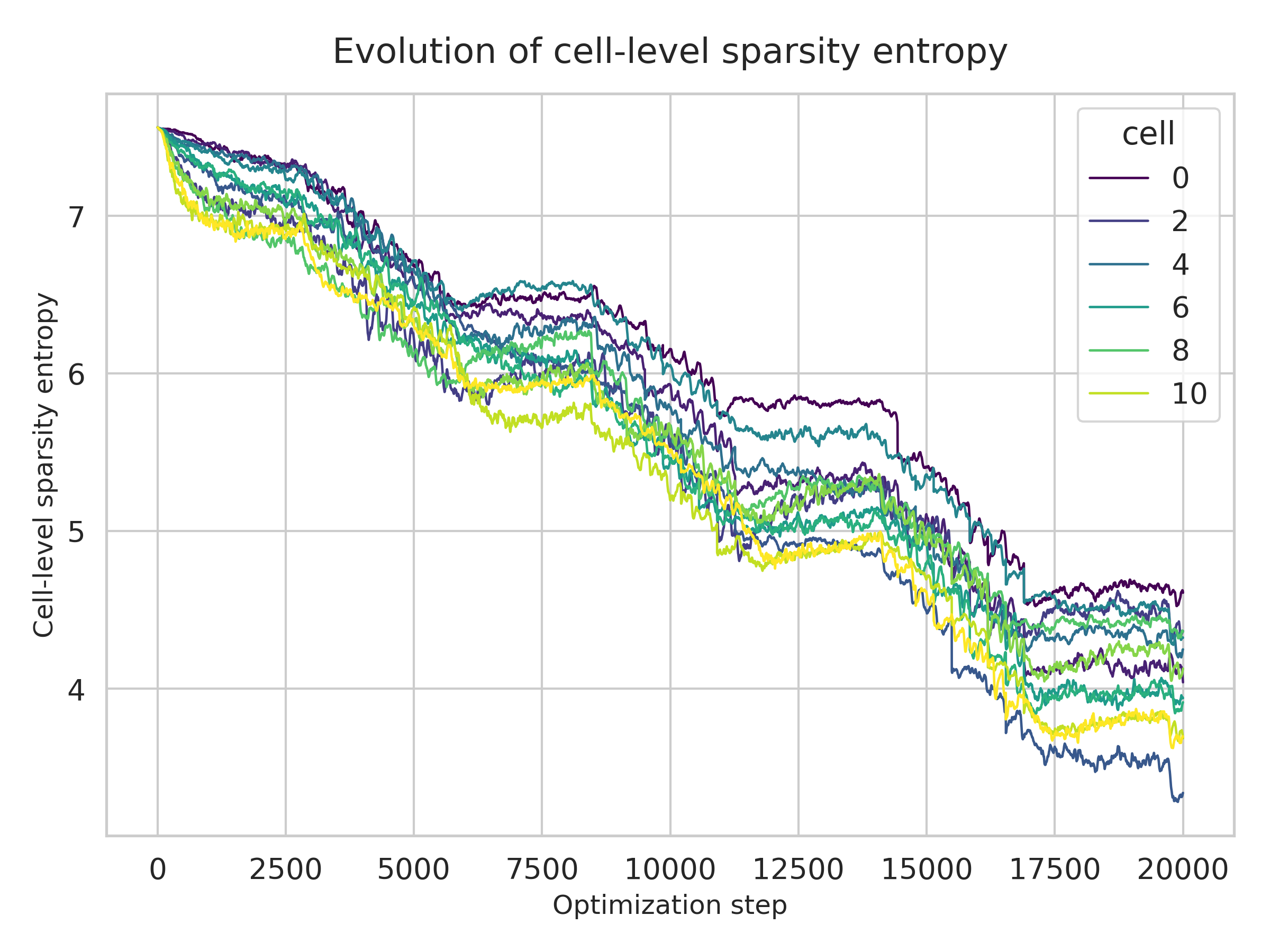}
	\caption{Evolution of cell-level sparsity entropy during the architecture search process. As shown, the gradually decreased sparsity entropy of each cell indicates that the super-network structure becomes increasingly sparse.}
	\label{fig:entropyhistory}
\end{figure}

The evaluation results are summarized in TABLE \ref{tab:results_CIFAR-10} and Fig. \ref{fig:nascomparison}. Partial FX-DARTS architectures are also visualized in Fig. \ref{fig:network_cifar}. Specifically, FX-DARTS-$\mathcal{O}_{1}$ (48E) refers to the architecture derived after 48 epochs (equivalent to $3\times T_{\rm search}$ when $T_{\rm search} = 16$) in operator space $\mathcal{O}_{1}$, and the other architectures follow the same naming convention. Several key observations are as follows.
\begin{itemize}
	\item FX-DARTS architectures consistently achieve competitive results on both CIFAR-10 and CIFAR-100 datasets. For instance, FX-DARTS-$\mathcal{O}_{1}$ (48-epoch), which achieves an average test accuracy of 78.32\% on CIFAR-100 and 95.95\% on CIFAR-10, outperforms all baseline architectures with constrained structures.
	
	\item FX-DARTS is able to derive competitive architectures across different operator spaces. For example, FX-DARTS-$\mathcal{O}_{1}$, FX-DARTS-$\mathcal{O}_{2}$ and FX-DARTS-$\mathcal{O}_{3}$ achieve high accuracy on CIFAR-10 and CIFAR-100 compared to other baseline architectures. 
	
	\item In addition, FX-DARTS is able to derive architectures with varying computational complexities within a single search procedure. For example, FX-DARTS-$\mathcal{O}_{1}$ (48E) has 1.06M parameters and 165M FLOPs, while FX-DARTS-$\mathcal{O}_{1}$ (80E) has 0.84M parameters and 133M FLOP.
\end{itemize}
What's more, Fig. \ref{fig:entropyhistory} illustrates the evolution of cell-level sparsity entropy during the architecture search process. As the search progresses, the sparsity entropy of each cell gradually decreases, indicating that the super-network structure becomes increasingly sparse.

\subsection{Multi-task evaluation}
\begin{table}[htb]
	\centering
	\caption{Datasets used in multi-task evaluations}
	\scalebox{0.9}{
		\begin{tabular}{lllll}
			\toprule
			\textbf{Dataset} & \textbf{Classes} & \textbf{Training Images} & \textbf{Test Images} & \textbf{Image Size} \\
			\midrule
			CIFAR-100        & 100             & 50,000                   & 10,000               & $32 \times 32$               \\
			TinyImageNet    & 200             & 100,000                  & 10,000               & $64 \times 64$               \\
			SVHN             & 10              & 73,257                   & 26,032               & $32 \times 32$               \\
			Flowers102       & 102             & 1,020                    & 6,149                & $64 \times 64$               \\
			\bottomrule
		\end{tabular}
	}
	\label{tab:dataset_info}
\end{table}
To further validate the generalization capability and robustness of FX-DARTS architectures across diverse visual recognition tasks, we conduct multi-task evaluations on four benchmark datasets including CIFAR-100, TinyImageNet, SVHN, and Flowers102. The key characteristics of these datasets are summarized in Table \ref{tab:dataset_info}. Our evaluation protocol maintains consistent training hyper-parameters with the CIFAR-10 experiments except that the batch size is 256 and the initial learning rate is 0.1. Data augmentation strategies include random cropping (20-100\% scale), horizontal flipping (50\% probability), color jittering (80\% probability), grayscale conversion (20\% probability), and ImageNet-standard normalization. 

Furthermore, we also conduct FX-DARTS on the TinyImageNet dataset. The search hyper-parameters are the same as those in CIFAR-10 experiments except that the images are resized to $64\times64$ and the training images are augmented by those the same as in the training phase. 

\begin{table*}[htbp]
	\centering
	\caption{Comparison results of NAS architectures on multi-task evaluations}
	\resizebox{\textwidth}{!}{
		\begin{tabular}{lccccccc}
			\toprule
			Model            & Params (M) & FLOPs (M) & CIFAR-100 (\%) & TinyImageNet (\%) & SVHN (\%) & Flowers102 (\%) & Overall (\%) \\
			\midrule
			Noisy-DARTS      & 3.3        & 381       & 78.06          & 62.31             & 96.75     & 66.38           & 82.99       \\
			XNAS             & 3.6        & 438       & 77.56          & 60.86             & 96.23     & 66.21           & 82.34       \\
			$\beta$-DARTS    & 3.5        & 419       & 77.65          & 62.37             & 96.41     & 67.28           & 82.86       \\
			CDARTS-CIFAR     & 3.7        & 441       & 77.67          & 60.22             & 96.35     & 66.09           & 82.28       \\
			DARTS-V2         & 3.3        & 389       & 77.28          & 61.73             & 96.35     & 67.38           & 82.65       \\
			PC-DARTS         & 3.6        & 394       & 77.44          & 61.21             & 96.57     & 67.91           & 82.79       \\
			P-DARTS          & 3.6        & 421       & 77.80          & 62.56             & 96.89     & 67.31           & 83.17       \\
			RelativeNAS      & 3.7        & 423       & 77.47          & 61.15             & 96.47     & 68.14           & 82.72       \\
			\midrule
			FX-DARTS-$\mathcal{O}_{1}$ (64E) & 4.0        & 472       & 78.01          & 62.83             & 96.36     & 68.69           & 83.21       \\
			FX-DARTS-$\mathcal{O}_{2}$ (64E) & 4.0        & 456       & 78.27          & 63.36             & 96.29     & 68.55           & 83.26       \\
			FX-DARTS-$\mathcal{O}_{3}$ (64E) & 4.2        & 481       & 78.67          & 63.52             & 96.16     & 68.19           & 83.29       \\
			FX-DARTS-$\mathcal{O}_{1}$ (Tiny, 64E) & 3.8        & 445       & 78.42          & 63.38             & 96.56     & 68.79           & 83.45       \\
			FX-DARTS-$\mathcal{O}_{2}$ (Tiny, 64E)& 3.7        & 440       & 78.38          & 63.25             & 96.32     & 68.45           & 83.25       \\
			FX-DARTS-$\mathcal{O}_{3}$ (Tiny, 64E)& 3.5        & 416       & 78.59          & 63.42             & 96.58     & 68.55           & 83.47       \\
			\bottomrule
		\end{tabular}
	}
	\label{tab:model_performance_updated}
\end{table*}
	
\begin{figure}
	\centering
	\includegraphics[width=1.05\linewidth]{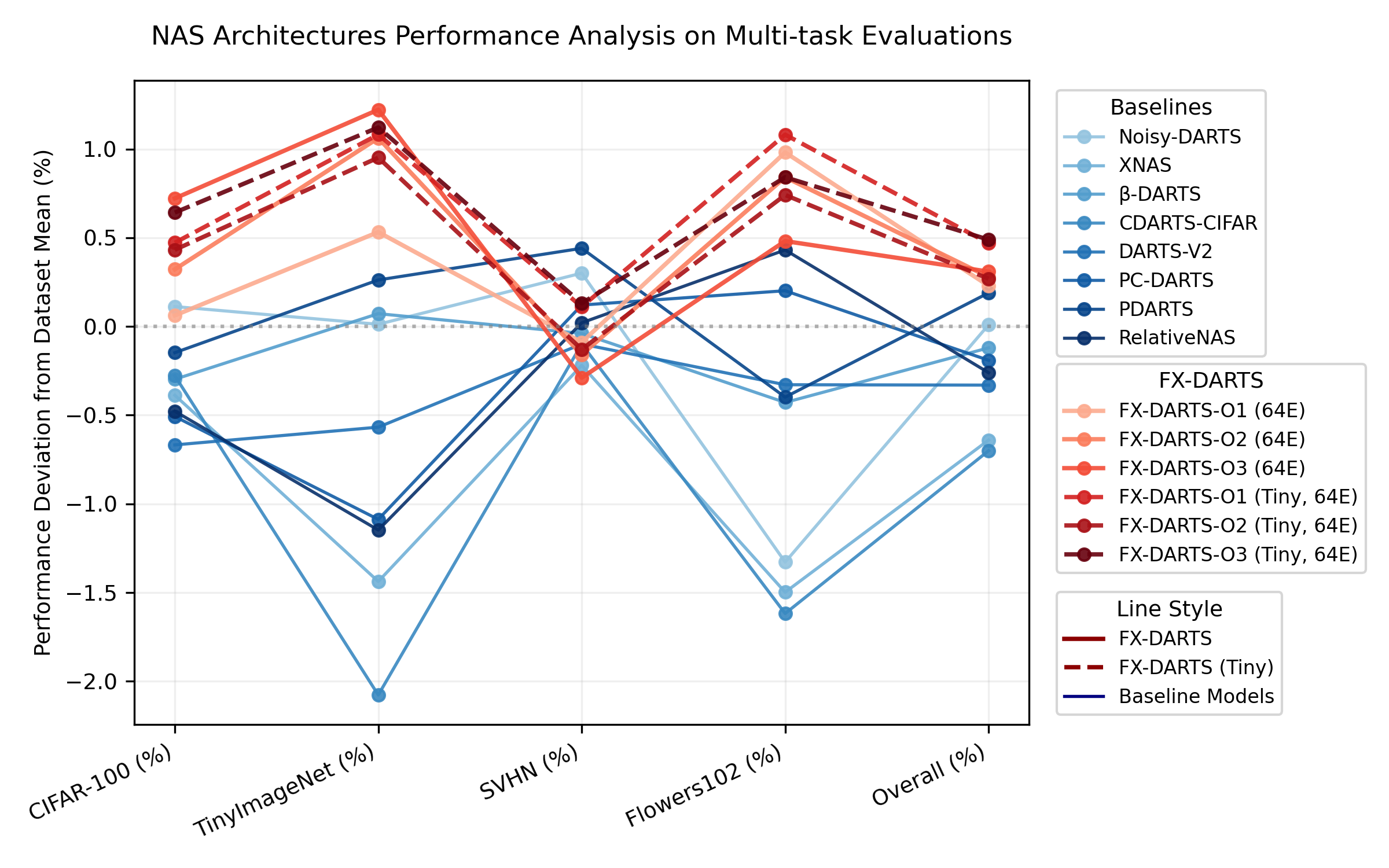}
	\caption{Performance comparison of FX-DARTS vs. Topology-constrained architectures on the multi-task evaluations. The results correspond to Table \ref{tab:model_performance_updated}. Specifically, The figure illustrates the relative accuracy improvement (in percentage points) compared to the mean accuracy of all models in each dataset.}
	\label{fig:multitaskanalysissinglecol}
\end{figure}

The results of the multi-task evaluation are summarized in TABLE \ref{tab:model_performance_updated}. Specifically, the table presents the performance of different architectures across the four datasets, along with the average accuracy across all tasks. In addition, Fig. \ref{fig:multitaskanalysissinglecol} presents the relative accuracy improvements of each model, which is calculated by subtracting the accuracy by the mean of all models in each datasets. Three key observations are as follows.

\begin{itemize}
	\item FX-DARTS architectures consistently outperform conventional topology-constrained NAS architectures across all evaluation benchmarks. The FX-DARTS-$\mathcal{O}_{3}$ (64E) architecture, which achieves the state-of-the-art results on CIFAR-100 (78.67\%) and TinyImageNet (63.52\%), demonstrates particularly strong performance on complex classification tasks requiring fine-grained discrimination.
	
	\item FX-DARTS models achieve significant accuracy improvements with comparable computational complexity. Notably, FX-DARTS-$\mathcal{O}_3$ (Tiny, 64E) attains a 0.86\% higher accuracy on TinyImageNet (63.42\% vs. 62.56\%) than P-DARTS (which is the optimal topology-constrained architecture) while reducing FLOPs by 1.1\% (415.91M vs. 420.69M).
	
	\item In addition, models searched directly on TinyImageNet exhibit enhanced task adaptability. For example, FX-DARTS-$\mathcal{O}_3$ (Tiny, 64E) achieves the highest overall accuracy (83.47\%) with 3.49M parameters and 416M FLOPs.
\end{itemize}

\subsection{Ablation study}
\begin{table*}[htbp]
	\centering
	\caption{Ablation study results of FX-DARTS components on multi-task evaluations}
	\label{tab:ablation_study}
	\begin{tabular}{lccccccc}
		\toprule
		Model & Params (M) & FLOPs (M) & CIFAR-100 (\%) & TinyImageNet (\%) & SVHN (\%) & Flowers102 (\%) & Overall (\%) \\
		\midrule
		one-level DARTS & 3.86 & 454 & 77.15 & 62.27 & 96.02 & 68.68 & 82.72 \\
		DARTS+AO+MR & 3.89 & 450 & 78.32 & 62.68 & 96.31 & 69.44 & 83.25 \\
		DARTS+AO+SNS & 3.82 & 451 & 78.06 & 62.77 & 96.41 & 69.13 & 83.23 \\
		FX-DARTS ($T=8$) & 3.66 & 434 & 77.88 & 62.85 & 95.93 & 67.51 & 82.87 \\
		FX-DARTS ($T=16$) & 3.80 & 445 & 78.42 & 63.38 & 96.56 & 68.79 & 83.45 \\
		FX-DARTS ($T=32$) & 3.65 & 441 & 78.54 & 63.56 & 96.36 & 69.80 & 83.53 \\
		FX-DARTS ($c=1.3/0.7$) & 3.63 & 437 & 77.94 & 62.80 & 96.25 & 68.01 & 83.08 \\
		FX-DARTS (2$\Delta E$) & 3.83 & 445 & 78.08 & 62.55 & 96.29 & 67.93 & 82.99 \\
		\bottomrule
	\end{tabular}
	
	\begin{tablenotes}
		\scriptsize
		\item AO: Architectural optimization, MR: Cycling model reinitialization, SNS: Super-network shrinking with adaptive coefficient adjustment.
		\item $T=16/32/8$: Search epoch settings of $T_{\rm search}=8/16/32$, $c=1.3/0.7$: Balance coefficients setting of $c_1 = 1.3$ and $c_{2} = 0.7$.
	\end{tablenotes}
\end{table*}

To validate the effectiveness of key components in FX-DARTS, we conduct comprehensive ablation experiments by progressively integrating the components of FX-DARTS into the DARTS framework. Table \ref{tab:ablation_study} summarizes the results. Summary of analysis is given as follows.

\begin{itemize}
	\item \textbf{The analysis on search epochs}: 
	FX-DARTS ($T=32$) achieves the 83.53\% overall accuracy, showing +0.66\% improvement over FX-DARTS ($T=8$) (82.87\%). However, the gain over FX-DARTS ($T=16$) ($83.45\%$) is marginal ($+0.08\%$), which suggest that $T_{\rm search}=16$ strikes an appropriate balance between search cost and architecture performance.
	
	\item \textbf{The impact of expected entropy reduction}: In particular, we double the expected entropy reduction for each step as $\Delta E \leftarrow 2\Delta E$ and verify its effects on the derived architectures. Specifically, the 2$\Delta E$ FX-DARTS variant ($82.99\%$) underperforms FX-DARTS ($T=16$) ($83.45\%$) by $-0.46\%$. It indicates that aggressive entropy minimization prematurely collapses the search space. In addition, FX-DARTS without adaptive entropy regularization (namely, DARTS+AO+SNS), is also very competitive ($83.24\%$ overall accuracy), but slightly lower than that of FX-DARTS ($T=16$) ($83.45\%$). In particular, the results show that the adaptive adjustment for sparsity regularization with moderate expected entropy reduction preserves diverse operator exploration while gradually focusing on promising candidates.
	
	\item \textbf{The sensitivity on shrinking coefficients}:
	FX-DARTS with $c_1=1.3$ and $c_2=0.7$ shows $0.37\%$ accuracy drop ($83.08\%$ vs. $83.45\%$) compared to default coefficients of $c_{1} = 1.05$ and $c_{2} = 0.95$, which indicates that drastic coefficient adjustments destabilize the balance between architecture exploration and exploitation. 
	
	\item \textbf{Component additions}: 
	DARTS+AO+MR ($83.25\%$) and DARTS+AO+SNS ($83.23\%$) both surpass the one-level DARTS baseline ($82.72\%$) by $+0.53\%$ and $+0.51\%$, respectively. This validates that architectural optimization synergizes with either model reinitialization or super-network shrinking with adaptive coefficient adjustment to prevent parameter co-adaptation.
	
	\item \textbf{ESS and MR comparison}:
	While MR slightly outperforms SNS ($83.25\%$ vs. $83.23\%$), their comparable gains suggest both strategies effectively address architecture overfitting. The full framework combines their strengths through phased implementation.
	
	\item \textbf{Full framework synergy}:
	FX-DARTS ($T=16$) achieves the best balance with $83.45\%$ accuracy, which is $+0.73\%$ over baseline and $+0.20\%$ over partial implementations. It demonstrates that coordinated component integration (AO, MR, and SNS) yields superior architectures.
\end{itemize}

\subsection{Experiments on ImageNet}
\subsubsection{Architecture search}
\begin{table*}[htbp]
	\centering
	\caption{Comparison with state-of-the-art neural architectures on ImageNet-1K}
	\begin{tabular}{lccccc}
		\toprule
		Architecture & Params (M) & FLOPs (M) & \multicolumn{2}{c}{Accuracy (\%)} & Search Cost \\
		\cmidrule(lr){4-5}
		& & & Top-1 & Top-5 & (GPU-time / GPU-Mem) \\
		\midrule
		Inception-V1 \cite{szegedy2015going} & 6.6 & 1448 & 69.8 & 89.9 & - \\
		ShuffleNet-V1 ($2\times$) \cite{DBLP:conf/cvpr/ZhangZLS18} & 5.4 & 524 & 73.6 & 89.8 & - \\
		ShuffleNet-V2 ($2\times$) \cite{DBLP:conf/eccv/MaZZS18} & 7.4 & 591 & 74.9 & 92.4 & - \\
		MobileNet-V1 \cite{howard2017mobilenets} & 4.2 & 575 & 70.6 & 89.5 & - \\
		MobileNet-V2 ($1.4\times$) \cite{DBLP:conf/cvpr/SandlerHZZC18} & 6.9 & 585 & 74.7 & - & - \\
		\midrule
		NASNet-A \cite{baker2016designing} & 5.3 & 564 & 74.0 & 91.6 & 1800 GPU-days / - \\
		AmoebaNet-A \cite{DBLP:conf/aaai/RealAHL19} & 5.1 & 555 & 74.5 & 92.4 & 3150 GPU-days / - \\
		DARTS-2nd \cite{liu2018darts} & 4.7 & 574 & 73.3 & 91.3 & 4.0 GPU-days / - \\
		PC-DARTS (CIFAR) \cite{xu2019pc} & 5.3 & 586 & 74.9 & 92.2 & 0.3 GPU-days / - \\
		PC-DARTS (ImageNet) \cite{xu2019pc} & 5.3 & 597 & 75.8 & 92.7 & 3.8 GPU-days / - \\
		P-DARTS \cite{chen2021progressive} & 4.9 & 557 & 75.6 & 92.6 & 0.3 GPU-days / - \\
		GDAS \cite{dong2019searching} & 5.3 & 581 & 74.0 & 91.5 & 0.2 GPU-days / - \\
		RandWire-WS \cite{DBLP:conf/iccv/XieKGH19} & 5.6 & 583 & 74.7 & 92.2 & - \\
		$\ell$-DARTS \cite{hu2024} &  5.8 & - & 75.1 & 92.4 & 0.06 GPU-days / - \\
		SWD-NAS \cite{xue2024self} & 6.3 & - & 75.5 & 92.4 & 0.13 GPU-days / - \\
		STO-DARTS-V2 \cite{cai2024sto} & 3.8 & - & 75.1 & 92.8 & - \\
		SaDENAS \cite{han2024sadenas} & 5.6 & - & 75.1 & 92.0 & - \\
		EG-NAS \cite{cai2024eg} & 5.3 & - & 75.1 & - & 0.1 GPU-days / - \\ 
		OLES \cite{jiang2024operation} & 4.7 & - & 75.5 & 92.6 & 0.4 GPU-days / - \\
		DARTS-PT-CORE \cite{xie2024darts} & 5.0 & - & 75.0 & - & 0.8 GPU-days / - \\
		
		\midrule
		FX-DARTS-$\mathcal{O}_{1}$ (64E) & 5.3 & 592 & 75.9 & 92.7 & 1.6 GPU-hours / 7.0G \\ 
		FX-DARTS-$\mathcal{O}_{3}$ (64E) & 5.2 & 560 & 76.0 & 92.8 & 2.7 GPU-hours / 11G \\
		FX-DARTS-$\mathcal{O}_{1}$ (Tiny, 64E) & 4.9 & 560 & 76.0 & 93.7 & 2.7 GPU-hours / 14G \\
		FX-DARTS-$\mathcal{O}_{3}$ (Tiny, 64E) & 4.9 & 575 & 76.2 & 93.0 & 3.0 GPU-hours / 20G \\
		FX-DARTS-$\mathcal{O}_{1}$ (Tiny, 128E, $T_{\rm search} = 32$) & 5.1 & 610 & 76.4 & 93.4 & 4.3 GPU-hours / 20G \\
		\bottomrule
	\end{tabular}
	\label{tab:results_imagenet-1K}
\end{table*}

In this section, we evaluate the CIFAR-10 and TinyImageNet-derived architectures on the ImageNet-1K dataset, which contains 1.3 million training images and 50,000 validation images. Specifically, we select the top-performing FX-DARTS architectures from TinyImageNet for evaluation. The architectures are trained for 250 epochs on four RTX 4090 GPUs in parallel. The initial number of network channels is set to 46. The model parameters are optimized using an SGD optimizer with an initial learning rate of \(0.5\), a weight decay of \(3 \times 10^{-4}\), a gradient clipping threshold of \(5.0\), a batch size of 1024, and a momentum of \(0.9\). The learning rate decays to zero following a linear scheduler during training. Additional enhancements include label smoothing with a rate of 0.1 and a learning rate warm-up applied to the first 5 epochs. Most settings follow those used in DARTS \cite{liu2018darts}. Table \ref{tab:results_imagenet-1K} presents the evaluation results, which reveal several key insights as

\begin{itemize}
	\item FX-DARTS-$\mathcal{O}_{3}$ (Tiny, 64E), which achieves 76.2\% top-1 accuracy with only 3.0 GPU-hours of search cost, outperforms both PC-DARTS (75.8\%) and MobileNet-V2 (74.7\%). It demonstrates our framework's ability to discover high-performance architectures in the enlarged search space while maintaining computational efficiency (575M FLOPs vs. 585M FLOPs baseline).
	
	\item Architectures discovered on smaller datasets exhibit strong transferability. For example, the FX-DARTS-$\mathcal{O}_{1}$ (Tiny, 64E) model, searched using TinyImageNet, achieves 76.0\% top-1 accuracy on ImageNet-1K, which is comparable to in-domain NAS results (PC-DARTS: 75.8\%) but with a shorter search time (2.7 GPU-hours vs. 3.8 GPU-days). 
	
	\item FX-DARTS-$\mathcal{O}_{1}$ (Tiny, 128E) demonstrates that the performance of FX-DARTS can be further improved (76.4\% accuracy, a +0.4\% increase over the 64E version) by doubling the search epochs while maintaining practical costs (4.3 GPU-hours). These experiments demonstrate that the representation performance of NAS architectures can be significantly enhanced under similar computational complexity with flexible architectures.
\end{itemize}

\section{Conclusions}
This paper presents FX-DARTS, a novel differentiable neural architecture search method that reduces prior constraints on cell topology and discretization mechanisms, to explore more flexible neural architectures. By eliminating the topology-sharing strategy and operator retention rules in conventional DARTS, FX-DARTS significantly expands the search space while addressing the resultant optimization challenges through the proposed Entropy-based Super-network Shrinking (ESS) framework. ESS which dynamically balances exploration and exploitation by adaptively adjusting sparsity regularization based on entropy reduction feedback enables the derivation of diverse architectures with varying computational complexities within a single search procedure. Extensive experiments on CIFAR-10/100, TinyImageNet, and ImageNet-1K demonstrate that FX-DARTS discovers architectures achieving competitive performance (e.g., 78.32\% on CIFAR-100, 76.4\% top-1 accuracy on ImageNet-1K) compared to state-of-the-art topology-constrained NAS methods, while maintaining efficient search costs (3.2 GPU-hours).

However, our current experiments have not demonstrated the overwhelming advantages of flexible architectures over the state-of-the-art architectures searched in the constrained space of DARTS. The possible reason is that the space enlargement prevents the NAS method from obtaining preferable architectures, as shown by ablation experiments which reveal unsatisfactory architectures resulting from the direct search of DARTS in the enlarged space. Overall, our future direction is to develop a more effective NAS method to derive more powerful architectures with flexible structures, regardless of the significant challenges posed by space enlargement and the absence of prior rules on discrete architectures.

\bibliographystyle{IEEEtran}
\bibliography{mybibfile}
\end{document}